\documentclass{article}
\usepackage[preprint, nonatbib]{neurips_2022}

\usepackage[utf8]{inputenc} % allow utf-8 input
\usepackage[T1]{fontenc}    % use 8-bit T1 fonts
\usepackage{url}            % simple URL typesetting
\usepackage{booktabs}       % professional-quality tables
\usepackage{amsfonts}       % blackboard math symbols
\usepackage{subfigure}
\usepackage{nicefrac}       % compact symbols for 1/2, etc.
\usepackage{microtype}      % microtypography
\usepackage{xcolor}         % colors
\usepackage{amsmath}
\usepackage{wrapfig}
\usepackage{graphicx}
\usepackage{amssymb}
\usepackage{algorithm}
\usepackage{algorithmic}
\usepackage{multirow}
\usepackage{diagbox}
\usepackage{booktabs}
\usepackage{ntheorem}
\usepackage{bm}
\usepackage{colortbl}
\usepackage{hyperref}
\hypersetup{colorlinks,
            linkcolor=red,
            anchorcolor=blue,
            citecolor=brown}
\newtheorem{theorem}{Theorem}

\newtheorem*{proof}{Proof}[section]
\newtheorem{assumption}{Assumption}

\title{DOTIN: Dropping Task-Irrelevant Nodes for GNNs}

\author{Shaofeng Zhang$^1$, Feng Zhu$^2$, Junchi Yan$^{1}$\thanks{Junchi Yan is the correspondence author. Shaofeng Zhang, Junchi Yan, Xiaokang Yang are with Department of Computer Science and Engineering, and MoE Key Lab of Artificial Intelligence, Artificial Intelligence Institute, Shanghai Jiao Tong University. Rui Zhao is also with Qing Yuan Research Institute, Shanghai Jiao Tong University. The work was partly supported by National Key Research and Development Program of China (2020AAA0107600) and Shanghai Municipal Science and Technology Major Project (2021SHZDZX0102).}, Rui Zhao$^{1,2}$, Xiaokang Yang$^1$\\
$^1$Shanghai Jiao Tong University, $^2$SenseTime Research\\
\texttt{\{sherrylone, yanjunchi, xkyang\}@sjtu.edu.cn}\\ \texttt{\{zhufeng, zhaorui\}@sensetime.com}
}
\begin{document}

\maketitle

\begin{abstract}
Scalability is an important consideration for deep graph neural networks. Inspired by the conventional pooling layers in CNNs, many recent graph learning approaches have introduced the pooling strategy to reduce the size of graphs for learning, such that the scalability and efficiency can be improved. However, these pooling-based methods are mainly tailored to a single graph-level task and pay more attention to local information, limiting their performance in multi-task settings which often require task-specific global information. In this paper, departure from these pooling-based efforts, we design a new approach called DOTIN (\underline{D}r\underline{o}pping \underline{T}ask-\underline{I}rrelevant \underline{N}odes) to reduce the size of graphs. Specifically, by introducing $K$ learnable virtual nodes to represent the graph embeddings targeted to $K$ different graph-level tasks, respectively, up to 90\% raw nodes with low attentiveness with an attention model -- a transformer in this paper, can be adaptively dropped without notable performance decreasing. Achieving almost the same accuracy, our method speeds up GAT about 50\% on graph-level tasks including graph classification and graph edit distance (GED) with about 60\% less memory, on D\&D dataset. Code will be made public available in \href{https://github.com/Sherrylone/DOTIN}{https://github.com/Sherrylone/DOTIN}.

%\yanr{issues to solve: 1) definition of task-irrelevant; difference to VCN; more }
\end{abstract}

\section{Introduction}
\vspace{-6pt}
% \yanr{Review comments:}

%\yanr{the strict definition of unimportant nodes, and why we need to emphasize it in our paper from the irrelevant nodes? If we want to emphasize the unimportant nodes? how our experiments validate this point?}

% \yanr{what is the difference to VCN work?}

% \yanr{why mean is an effective method? or is there any other means effective or more effective than mean?}

In recent years, there has been a surge of interest in developing graph neural networks (GNNs) to extract semantic features of graph-structured data, such as social network data~\cite{GIN, GCN, GAT} and graph-based representations of molecules~\cite{molecules}. The success behind GNNs (against MLPs) is message passing between nodes by using (task-)specific prior knowledge of node adjacency~\cite{mpassing}. However, the message passing also brings $\mathcal{O}(N^2)$ complexity (each node computes the weighted average embeddings of its neighbors). To reduce the complexity and enhance the model's scalability, for graph-level tasks, one plausible solution is narrowing down the graph size hierarchically, which can be fulfilled by graph pooling, as the counterpart to that in CNNs~\cite{CNN}.

Graph pooling methods~\cite{diffpooling, pooling1, topk, SAPool} usually stack GNN and pooling layers, which extract local (sub-graph) representations from several neighbor nodes. Previous pooling methods mainly differ in how to divide sub-graphs~\cite{diffpooling, pooling1} and assign pooling weights of nodes~\cite{topk, pooling1}. Then, the extracted sub-graph representations are fed into the next GNN layer to reduce the overall graph size. The whole architecture of the pooling-based GNN is akin to CNNs, which uses convolution layers as filters and pooling layers to increase the receptive field. However, these pooling methods show defects in two aspects. \textbf{i) structure design}, pooling methods over-emphasize the local information~\cite{vit}, but ignore global interaction which is useful for down-stream tasks (classification~\cite{vit}); \textbf{ii) learning paradigm}, pooling schemes are commonly used in GNNs which are task-agnostic, since they don't use the global representation to select nodes to drop, and correspondingly, they can \textbf{not} learn which nodes are \textit{task-irrelevant} for specific tasks, as illustrated in Fig.~\ref{fig:illustration} and in our later ablation studies.

In this paper, inspired by the recent transformer-based methods~\cite{vit, evit}, we propose DOTIN to tackle the above two problems. Specifically, DOTIN uses $K$ virtual nodes to directly capture global information targeted to $K$ different tasks (each virtual node learns one task-specific global information). Then, the virtual nodes compute the mean attentiveness of the $K$ tasks, and adaptively select to drop task-irrelevant nodes over the whole graph. In this paper, task-irrelevant nodes are defined as those with lower attentiveness by the softmax based attention mechanism as will be shown in the technical part of the paper. Similar meaning for edges are also termed in~\cite{dropedge2}. The highlights of the paper are:
%\footnote{Task-irrelevant means less useful for the given task, similar defined in~\cite{dropedge2}.}
\begin{wrapfigure}{r}{0.42\textwidth}
\vspace{-15pt}
  \begin{center}
    \includegraphics[width=0.44\textwidth]{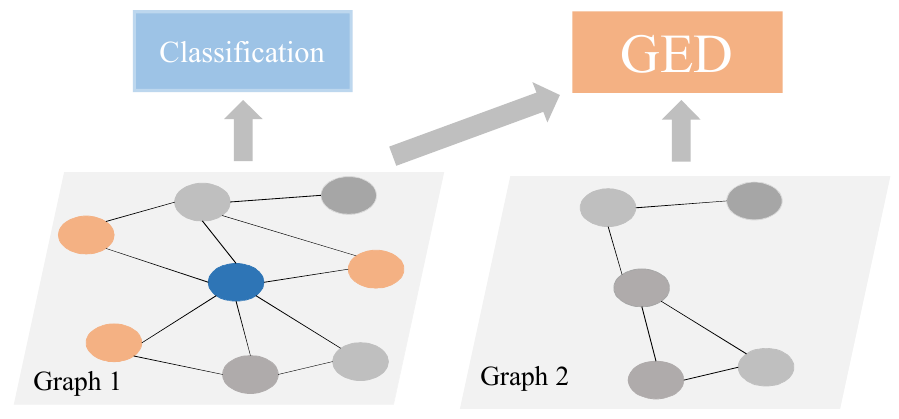}
  \end{center}
  \vspace{-10pt}
  \label{fig:illustration}
  \caption{Motivation of multiple virtual nodes setting. For single-graph tasks like graph classification, we may emphasize the center node (blue) of the graph. However, for pair-graph tasks like graph matching or other comparison tasks, the three orange nodes can be more important.}
  \vspace{-15pt}
\end{wrapfigure}

\textbf{1) }We show the phenomenon, i.e., in existing GNNs architecture, for different tasks, \textit{task-irrelevant} nodes may be similar but not the same (illustrated in Fig.~\ref{fig:illustration}). The proposed DOTIN can find those task-irrelevant nodes in both single-task and multi-task settings. As far as we know, DOTIN is the first for proposing using virtual nodes for multi-task learning.

% (learn local information hierarchically)
\textbf{2) }To our best knowledge, this is a new paradigm (agnostic to the choice of GNNs architecture e.g. GAT~\cite{GAT}, GCN~\cite{GCN}) to reduce the graph size by dropping nodes hierarchically, while previous methods almost adopt graph pooling. Besides, DOTIN only incurs $\mathcal{O}(N \log N)$ complexity by introducing and sorting $K$ vectors without any other parameters. In contrast, previous methods usually require extra clustering~\cite{diffpooling}, or GNN~\cite{mincut, diffpooling} and GCN~\cite{sagpool} layers that will bring much more time and space complexity than sorting.

\textbf{3) }We conduct experiments on both single-task and multi-task settings on benchmarks. Results show that our methods outperform peer pooling methods. In particular, DOTIN drops 90\% nodes without performance drop against the baseline and achieves about 50\% training speed gain on D\&D dataset.

\vspace{-6pt}
\section{Related Work}
\vspace{-6pt}
Here we briefly discuss the relevant works on cost-efficient GNN design and computing by edge/node drop and effective pooling. More comprehensive review is given in appendix.

\textbf{Graph pooling. }The mainstream of graph pooling methods can be divided into global and hierarchical approaches~\cite{rethinking1}. Global methods aggregate all nodes' representations either via simple flatten schemes, such as summation and average~\cite{GCN, GAT}. While hierarchical approaches coarsen graph representations layer-by-layer. DiffPool~\cite{diffpooling} is the seminal work of hierarchical approach, which downsamples graphs by clustering nodes in input graphs, and computes the assignment matrix in the $l$-th layer $\mathcal{S}^{(l)}$ with learned clusters. Specifically, nodes in the $l$-th layer are assigned by:
\begin{small}
\begin{equation}
    \mathbf{S}^{(l)} = \texttt{Softmax}\left(\texttt{GNN}_{l}(\mathbf{A}^{(l)}, \mathbf{X}^{(l)})\right), \ \ \mathbf{A}^{(l+1)} = \mathbf{S}^{(l) \top} \mathbf{A}^{(l)}\mathbf{S}^{(l)}
\end{equation}
\end{small}
where $\mathbf{X}^{(l)}$ and $\mathbf{A}^{(l)}$ are the node features and adjacency matrix of the $i$-th layer. Hence, a clustering complexity will be introduced. gPool~\cite{topk} designs a learnable vector to choose nodes to be retained. However, there's no regularization on the learnable vector, which may inappropriately delete \textit{task-relevant} and important nodes. SAGPool~\cite{sagpool} addresses this issue by using a graph convolution layer, followed by Sigmoid function to learn which nodes should be masked. Nevertheless, it introduces a GNN layer parameters and may ignore task-relevant information. 

\textbf{Edge/node drop. }Edge drop is often adopted as a regularization technique in GNNs, especially for preventing from over-smoothing and for better generalization~\cite{dropedge,dropedge2}. However, they can rarely help improve the scalability of the GNN model as the node size does not change. Accordingly, node drop is recently considered in DropGNN~\cite{dropgnn}, which is in fact the only node dropping work so far we have identified. DropGNN proposes to drop nodes randomly several times and ensemble each predictive results. However, by the DropGNN learning paradigm, some important nodes may be dropped, which limits their performance and stability. Besides, multiple running and ensemble the prediction of each dropped graph also increase the complexity and training time. Inspired by their work, give one or multiple specific task(s), it is more attractive to develop adaptive mechanism to select the task-irrelevant nodes for dropping in one shot.

%\textbf{2) } Virtual-node-based methods are in general inspired by token-based methods in vision~\cite{vit, virtual} and language~\cite{bert, transformer}. Specifically, VCN~\cite{graph-virtual} proposes to use a virtual node to extract task-relevant information. Specifically, the graph-level task objective is directly added to the embedding of the virtual node, forcing it to learn task-relevant information from input graphs.

\begin{figure}[tb!]
    \centering
    \includegraphics[width=0.9\textwidth]{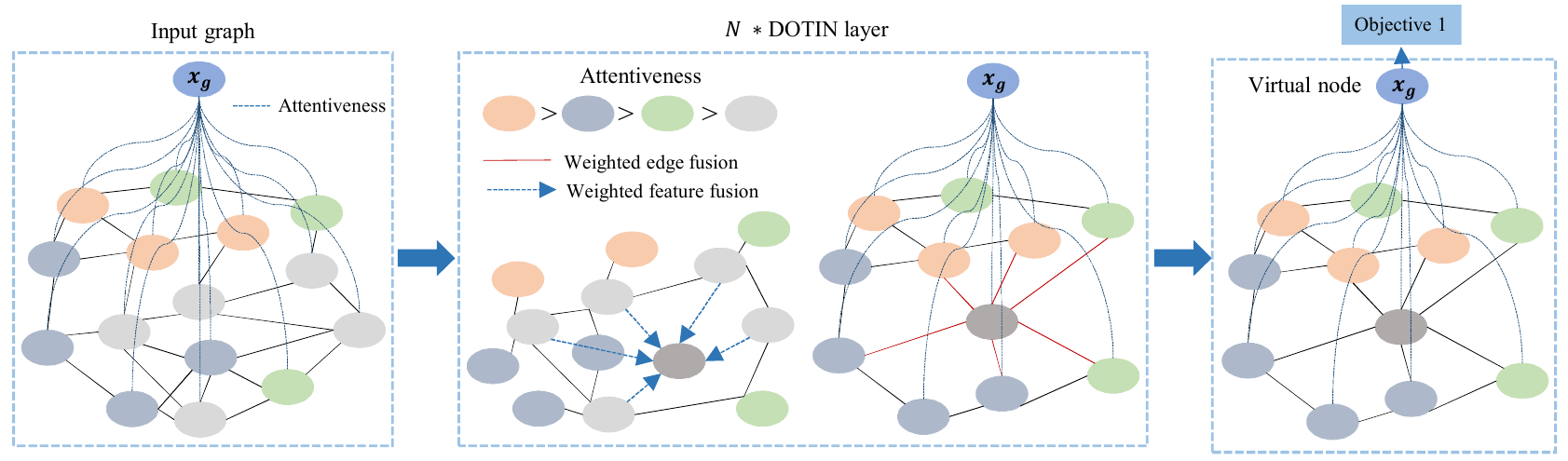}
    \vspace{-10pt}
    \caption{Framework of DOTIN for single task setting. The global virtual node is initially fully connected. Then five nodes in dash gray are fused in one global sub-global nodes. The new node (fused) reserves both structure (edges in red line) and statistic information (feature fusion). }
    \label{fig:singletask}
    \vspace{-10pt}
\end{figure}
\begin{figure}[tb!]
    \centering
    \includegraphics[width=0.9\textwidth]{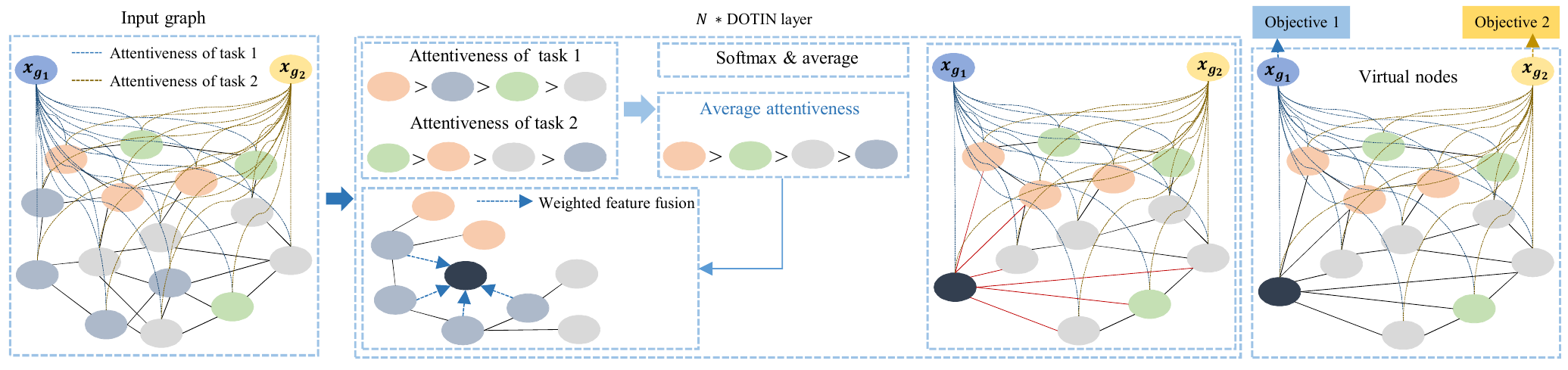}
    \vspace{-10pt}
    \caption{Framework of DOTIN for multiple-task setting (two tasks). Both global virtual nodes are initially fully-connected with Gaussian distribution. Then the mean attentiveness is used for replacing the attentiveness in each task. We also use one new node to represent the dropped sub-graph, which is similar to single task in Fig.~\ref{fig:singletask}. The overall objective is the sum of two single-task objectives.}
    \label{fig:multitask}
    \vspace{-10pt}
\end{figure}

\vspace{-6pt}
\section{Methodology}
\vspace{-6pt}

\subsection{Preliminaries}
\vspace{-6pt}
%Without loss of generality
%For concrete discussion, we here use GAT~\cite{GAT} as the backbone for our method. \yan{(In fact, our e will also compare GCNs~\cite{GCN} as backbone.)} 

We first briefly review the propagation mechanisms as used in GNNs regardless the specific backbone choice e.g. GAT~\cite{GAT} or GCNs~\cite{GCN} and we will compare them in our experiments. Denote one input graph as $\mathcal{G} = \{\mathbf{X}^{(0)}, \mathbf{A}^{(0)}\}$, where $\mathbf{X}^{(0)} \in \mathbb{R}^{N \times F}$ and $\mathbf{A}^{(0)} \in \mathbb{R}^{N \times N}$ are input feature matrix and adjacency matrix, respectively. Graph convolution propagates graph by:
\begin{small}
\begin{equation}
\label{eqn:GCN}
    \mathbf{X}^{(l+1)} = \sigma\left((\mathbf{D}^{(l)})^{-1/2} \mathbf{\hat{A}}^{(l)} (\mathbf{D}^{(l)})^{-1/2} \mathbf{X}^{(l)} \theta^{(l)}\right), \ \ \mathbf{A}^{(l+1)} = \mathbf{A}^{(l)},
\end{equation}
\end{small}
where $\mathbf{X}^{(l)}$ and $\mathbf{A}^{(l)}$ is the feature and adjacency matrix in the $l$-th layer, respectively. $\mathbf{\hat{A}}^{(l)} \in \mathbb{R}^{N \times N}$ is the adjacency matrix with self-loop. $\mathbf{D}^{(l)}$ is the degree matrix of $\mathbf{\hat{A}}^{(l)}$. $\theta^{(l)}$ is the learnable linear parameters of the $l$-th layer. Rather than use pre-defined edge weights, GAT tends to learn weights by itself, and the weights are computed by:
\begin{small}
\begin{equation}
    \mathbf{S}^{(l)} = \texttt{Softmax}\left((\mathbf{X}^{(l)} \mathbf{W}_1^{(l)} (\mathbf{X}^{(l)} \mathbf{W}_2^{(l)})^{\top}) \odot \mathbf{\hat{A}}\right)
\end{equation}
\end{small}
where $\odot$ is the element-wise product, and $\mathbf{W}^{(l)}$ are the learnable weights of the $l$-th layer.

\vspace{-6pt}
\subsection{The Proposed DOTIN}
\label{sec:method}
\vspace{-6pt}
The key idea behind DOTIN is adding virtual nodes. Then, the graph-level objectives ($K$ objectives for $K$ tasks, respectively) regularize the virtual nodes one by one, forcing them to learn task-relevant information. Then, by calculating the average attentiveness of the virtual nodes and raw nodes, we know which nodes are task-irrelevant to these tasks.

\textbf{Single-task virtual node. }The input feature matrix $\mathbf{X} \in \mathbb{R}^{N \times F}$ is firstly passed by a linear projection layer $\mathbf{X}^{(0)} = \mathbf{X}^{(0)}\mathbf{W}$, where $\mathbf{W} \in \mathbb{R}^{F \times D}$ and $D$ is the hidden dimension. Then a virtual node is initialized as a learnable vector, which is connected to all the other nodes in the graph. Then, the input feature matrix and matrix becomes $\mathbf{X}^{(0)} \in \mathbb{R}^{(N+1) \times D}$ and $\mathbf{A}^{(0)} \in \mathbb{R}^{(N+1) \times (N+1)}$, respectively, where the first raw feature of $\mathbf{X}^{(0)}$ is the virtual node embedding. Then, the propagation is in line with the applied backbones (e.g. GCNs, GATs etc.). For the last layer, we directly add the objective regularization on the virtual nodes. For GATs-based backbone, the propagation of global virtual node and objective are as follows, respectively:
\begin{small}
\begin{equation}
    \mathbf{x}_{g}^{(l+1)} = \sum_i \left( \frac{\exp(\mathbf{x}_g^{(l)\top} \mathbf{x}_i^{(l)})}{\sum_j \exp( \mathbf{x}_g^{(l)\top} \mathbf{x}_j^{(l)})} \cdot \mathbf{x}_{i}^{(l)} \right), \ \ 
    \mathcal{J} = \mathcal{L}_{cls, reg} (f(\mathbf{x}_{g}), y),
\end{equation}
\end{small}
where $\mathbf{x}_{g}$ is the embedding of virtual node (global embedding) and $\mathbf{x}_i$ is the $i$-th node of input graph. $y$ is the graph-level task label, and $f$ is the last linear layer to apply to different down-stream tasks.

\textbf{Multi-task virtual nodes. }For $K$ different graph-level tasks, we initialize $K$ learnable global virtual nodes with random distribution, where each virtual nodes are connected with the whole graph. The propagation rule is the same with single-task setting in hidden layers. Then, for the last layer, the propagation and objectives become to:
\begin{small}
\begin{equation}
    \mathbf{x}_{g_k}^{(l+1)} = \sum_i \left( \frac{\exp(\mathbf{x}_{g_k}^{(l)\top} \mathbf{x}_i^{(l)})}{\sum_j \exp( \mathbf{x}_{g_k}^{(l)\top} \mathbf{x}_j^{(l)})} \cdot \mathbf{x}_{i}^{(l)} \right), \ \ 
    \mathcal{J} = \frac{1}{K} \sum_i \mathcal{L}_{cls, reg} (f(\mathbf{x}_{g_i}), y_i),
\end{equation}
\end{small}
where $\mathbf{x}_{g_i}$ and $y_i$ are the global representation and ground truth for the $i$-th task.

\begin{assumption}
\label{assum}
Given one task-specific global representation $\mathbf{x}_{g}$ of a graph, if the nodes embedding $sim(\mathbf{x}_{1}, \mathbf{x}_g)>sim(\mathbf{x}_{2}, \mathbf{x}_g)$, we assume node $\mathbf{x}_1$ is more task-relevant than $\mathbf{x}_2$.
\end{assumption}

\textbf{Node dropping by attentiveness. }Assume in the $l$-th layer, there are $N+K$ nodes in input graph (including $K$ virtual nodes). Give a pre-defined drop ratio $0 < \alpha^{(l)} < 1$ for layer $l$, there are $M=\lceil N\cdot (1-\alpha) + K \rceil + 1$ nodes remained after dropping the raw nodes, including the $K$ virtual nodes and one sug-global node, which we will describe in detail below. Inspired by EViT~\cite{evit}, we calculate the attentiveness of virtual nodes as well as the remaining raw nodes by:
\begin{small}
\begin{equation}
\label{eqn:attentiveness}
    s_{i} = \sum_k (\mathbf{x}_{g_k} \mathbf{W}_1)^{\top} (\mathbf{x}_{i} \mathbf{W}_2),\ \  \mathbf{s} = \texttt{Softmax}(\mathbf{s}/ \tau), \ \ 0 < i \leq N
\end{equation}
\end{small}
where $\mathbf{W}_1$ and $\mathbf{W}_2$ are two learnable parameters in GATs~\cite{GAT}, $\tau$ is set $D^{1/2}$ by default. After obtaining the attentiveness of all the remaining raw nodes, by Assumption~\ref{assum}, we sort the attentiveness and find the index of the smallest $\lfloor N\cdot \alpha \rfloor$ number of nodes, which are called \textit{task-irrelevant} nodes in this paper. Then we introduce a new node to represent the sub-graph composed of these task-irrelevant nodes. The feature of the new node is obtained by mixing the task-irrelevant nodes with different weights (attentiveness), i.e.,
\begin{small}
\begin{equation}
\label{eqn:reweight}
    \mathbf{x}_{\lceil N\cdot (1-\alpha) \rceil + 1}' = \sum_i^{\lfloor N\cdot \alpha \rfloor} \lambda_i \cdot \mathbf{x}_i, \ \ \mathbf{\lambda} = \texttt{Softmax}([s[1], \cdots, s[\lfloor N\cdot \alpha \rfloor])
\end{equation}
\end{small}
To further reserve the structural information of the sub-graph, the new node connects all nodes connected with the task-irrelevant sub-graph and the edge weights is modified by:
\begin{small}
\begin{equation}
    w_{i, \lceil N\cdot (1-\alpha) \rceil + 1}' = \sum_j^{\lfloor N\cdot \alpha \rfloor} w_{i, j},\ \ \mathbf{w}_{i, :} = \texttt{Softmax}(\mathbf{w}_{i, :})
\end{equation}
\end{small}
where $\mathbf{w}_{i, :}$ is the vector composed of edge weights of node $i$ and the other nodes. This design is motivated by: 1) if two or more nodes in the task-irrelevant sub-graph connect nodes in the reserved graph, we think the node is more related to the task-irrelevant sub-graph, and the weight of new node with the node in reserve graph is the summation followed by a softmax regularization, which is more comprehensible than average (ignore the connection degree) and counting (ignore edge weights). Finally, we stack the propagation layers and node dropping stage one by one.

% \vspace{-5pt}
% \subsection{Empirical analysis}
% \vspace{-5pt}
\textbf{Backbone choose. }Here, we briefly analyze why we prefer to choose GAT but not GCNs (or other methods with pre-defined edge weights, e.g., GIN~\cite{GIN} and GraphSAGE~\cite{graphsage}) as the backbone. On one hand, the attentiveness in Eq.~\ref{eqn:attentiveness}, requires two linear layers and a dot product operation, which are also components in GAT layer. Nevertheless, GCN layer doesn't include these operations, which makes DOTIN less efficient in GCN. On the other hand, GCN-based backbones are also unsuitable for multi-task settings, which we will explain in detail below. The forward propagation rule is given in Eq.~\ref{eqn:GCN}, where $\mathbf{A}$ is updated in each layer. Consider $K$ tasks with $\alpha$ drop ratio in a graph including $N$ nodes, the output of DOTIN contains $K + 1 + \lfloor N \cdot (1 - \alpha) \rfloor$ nodes. Then, for the next layer, since each virtual node is initially fully-connected, their embedding is modified by:
\begin{small}
\begin{equation}
    \mathbf{x}_{g_k} = \frac{1}{2 + \lfloor N \cdot (1-\alpha) \rfloor} \left ( \sum_i \mathbf{x}_i + \mathbf{x}_{g_k}\right)\ \ \ for\ \ 1 \leq k \leq K
\end{equation}
\end{small}
We can observe the embeddings of virtual nodes are equivalent, which is contradictory to our initial desire, i.e., applying different virtual nodes to different tasks.

\vspace{-5pt}
\subsection{Theoretical analysis}
\vspace{-5pt}
DOTIN does not directly drop the task-irrelevant sub-graph,  instead it constructs a new node to reserve the structural information. Here, we theoretically analyze the motivation of our design.

\begin{theorem}
\label{the:higher}
\textbf{Attentive pooling for improving task-relevant node embedding. }Define the distance of two mode feature vectors by $l_2(\mathbf{x}_1, \mathbf{x}_2)=\Vert \mathbf{x}_1-\mathbf{x}_2 \Vert_2^2$. Given a graph learned with a given graph-level task $\mathcal{T}$ e.g. graph classification with node embeddings $\{\mathbf{x}_i\}_{i=1}^{N}$ of its sub-graph and the global graph-level embedding $\mathbf{x}_{g}$ learned from the given task, where $N$ is the number of nodes in the graph. Then, the nodes $\mathbf{x}_{new}$ produced by attentive pooling has closer distance to $\mathbf{x}_{g}$ than the average distance of the sub-graph before pooling, i.e., $l_2(\mathbf{x}_g, \mathbf{x}_{new}) < \frac{1}{\lceil N \cdot \alpha \rceil} \sum_i l_2(\mathbf{x}_g, \mathbf{x}_i)$.
\end{theorem}

\begin{proof}
To prove the theorem, we first prove the average pooling can reduce the distance, and in the second step, we prove the weighted by attentiveness pooling will construct a new node with larger attentiveness. By the definition, we have:
\begin{small}
\begin{equation}
\begin{aligned}
    \frac{1}{\lceil N \cdot \alpha \rceil} \sum_i & \left\|\mathbf{x}_g  \mathbf{W}_1 - \mathbf{x}_i \mathbf{W}_2 \right\|_2^2 = \frac{1}{\lceil N \cdot \alpha \rceil} \sum_i\left(\left\| \mathbf{x}_g \mathbf{W}_1 \right\|_2^2 + \left\| \mathbf{x}_i \mathbf{W}_2 \right\|_2^2 - 2 \cdot (\mathbf{x}_g \mathbf{W}_1)^{\top} (\mathbf{x}_i \mathbf{W}_2)\right) \\
    & = \left\| \mathbf{x}_g \mathbf{W}_1 - \frac{1}{\lceil N \cdot \alpha \rceil} \sum_i \mathbf{x}_i \mathbf{W}_2 \right\|_2^2 - \left\| \frac{1}{\lceil N \cdot \alpha \rceil} \sum_i \mathbf{x}_i \mathbf{W}_2 \right\|_2^2 + \frac{1}{\lceil N \cdot \alpha \rceil} \sum_i \left\| \mathbf{x}_i \mathbf{W}_2 \right\|_2^2 \\
    & = \left\| \mathbf{x}_g \mathbf{W}_1 -\frac{1}{\lceil N \cdot \alpha \rceil} \sum_i \mathbf{x}_i \mathbf{W}_2 \right\|_2^2 + \frac{1}{\lceil N \cdot \alpha \rceil} \sum_i \left \| \frac{1}{\lceil N \cdot \alpha \rceil} \sum_j \mathbf{x}_j \mathbf{W}_2 - \mathbf{x}_i \mathbf{W}_2 \right\|_2^2
\end{aligned}
\end{equation}
\end{small}
where the first term equals to $l_2(\mathbf{x}_g \mathbf{W}_1, avg(\mathbf{x}_i) \mathbf{W}_2)$ and the second term equals to $avg(l_2(\mathbf{x}_i \mathbf{W}_2, avg(\mathbf{x}_j) \mathbf{W}_2))$, which are thus both non-negative. For shallow GNNs, the node features are usually different, leading the second term R.H.S larger than 0. Then, we derive $\frac{1}{\lceil N \cdot \alpha \rceil} \sum_i \Vert\mathbf{x}_g \mathbf{W}_1 - \mathbf{x}_i \mathbf{W}_2 \Vert_2^2 \geq \Vert \mathbf{x}_g \mathbf{W}_1 -\frac{1}{\lceil N \cdot \alpha \rceil} \sum_i \mathbf{x}_i \mathbf{W}_2 \Vert_2^2$, i.e., the averagely constructed node has closer distance than the average distance of the task-irrelevant sub-graph. Then, we begin proving the weighted pooling by attentiveness will bring higher attentiveness. As defined in Eq.~\ref{eqn:attentiveness}, The attentiveness is calculated by $s_i = (\mathbf{x}_g \mathbf{W}_1)^{\top} (\mathbf{x}_i \mathbf{W}_2)$ followed by a softmax. For average sub-graph pooling, we have $s = (\mathbf{x}_g \mathbf{W}_1)^{\top} (\frac{1}{\lceil N \cdot \alpha \rceil} \sum_i \mathbf{x}_i \mathbf{W}_2) = \frac{1}{\lceil N \cdot \alpha \rceil} \sum_i (\mathbf{x}_g \mathbf{W}_1)^{\top} (\mathbf{x}_i \mathbf{W}_2)$, while for weighted sub-graph pooling, we have $s' = (\mathbf{x}_g \mathbf{W}_1)^{\top} (\mathbf{x}_{new} \mathbf{W}_2) = \sum_i (\mathbf{x}_g \mathbf{W}_1)^{\top} (\lambda_i \mathbf{x}_i \mathbf{W}_2)$. For $(\mathbf{x}_g \mathbf{W}_1)^{\top} (\mathbf{x}_{i} \mathbf{W}_2)>(\mathbf{x}_g \mathbf{W}_1)^{\top} (\mathbf{x}_{j} \mathbf{W}_2)$, we must have $\lambda_i > \lambda_j$ (attentiveness definition), while by the regularization in Eq.~\ref{eqn:reweight}, we have $\sum_i \lambda_i = 1$. Thus, we derive $s'>s$ and complete the proof.
\end{proof}

Theorem~\ref{the:higher} clarifies the motivation of task-irrelevant pooling methods, i.e., the newly constructed node has higher attentiveness and closer distance to the global virtual nodes. In other words, the newly constructed node is more task-relevant.

\begin{table}[tb!]
    \centering
    \resizebox{0.9\textwidth}{!}{\begin{tabular}{c|c|c|c|c|c}
       % \toprule
        \toprule
        ~ & Time & Space & Task-relevant & Reduce method~\cite{understanding} & Extra parameters \\
        \hline
        Set2Set~\cite{set2set} & LSTM & LSTM & No & $\mathbf{X}'=\mathbf{S}\cdot GNN(\mathbf{A}, \mathbf{X})$ & LSTM layer \\
        DiffPool~\cite{diffpooling} & GNN + CLuster & GNN & No & $\mathbf{X}'=\mathbf{S}\cdot GNN(\mathbf{A}, \mathbf{X})$ & GNN layer \\
        MinCut~\cite{mincut} & Cluster + GNN + MLP & GNN + MLP & No & $\mathbf{X}'=\mathbf{S}^{\top} \mathbf{X}$ & MLP + GNN layer \\
        LaPool~\cite{lapool} & Cluster + Linear & Linear & No & $\mathbf{X}'=\mathbf{S}^{\top} \mathbf{X}$ & Linear layer \\
        gPool~\cite{topk} & $\mathcal{O}(N \log N)$ & $\mathcal{O}(N)$ & No & $\mathbf{X}'=(\mathbf{X}\odot \sigma (\mathbf{y}))$ & GNN layer \\
        SAGPool~\cite{sagpool} & GCN + $\mathcal{O}(N \log N)$ & GCN & No & $\mathbf{X}'=(\mathbf{X}\odot \sigma (\mathbf{y}))$ & GCN layer \\
        DOTIN (Ours) & $\mathcal{O}(N \log N)$ & $\mathcal{O}(1)$ & Yes & $\mathbf{X}'=(\mathbf{X}\odot \sigma (\mathbf{y}))$ & $\mathcal{O}(1)$ \\
        %\bottomrule
        \bottomrule
    \end{tabular}}
    \caption{Methodology comparison. The time and space complexity are relevant to clustering iteration, GNN, MLP layers and their dimensions. gPool~\cite{topk} and DOTIN (ours) only use an extra rank time complexity. Space $\mathcal{O}(1)$ means we only use $K$ global vectors ($K=1$ when single-task setting). Task-relevant means whether the method is task-specific or not (extract multiple global embeddings).}
    \vspace{-15pt}
    \label{tab:method}
\end{table}
\vspace{-5pt}
\subsection{Methodology discussion}
\vspace{-5pt}
We compare DOTIN with previous pooling-based methods in Table~\ref{tab:method}. Here, we further clarify the connection and differences with similar methods gPool~\cite{topk} and SAGPool~\cite{sagpool}. \textbf{Connection: }All the three methods use attention methods to calculate the importance and adaptively choose task-irrelevant nodes to be dropped. \textbf{Differences: }The main difference between DOTIN and them is that DOTIN can recognize \textbf{task-irrelevant} nodes and the \textit{task-irrelevant} sub-graph pooling also makes DOTIN achieve higher accuracy. 
% Besides, DOTIN learns variant global representations targeted to different tasks, while gPool and SAGPool are more suitable for single-task settings. 
Moreover, DOTIN only introduces $K$ extra vectors while the previous two methods (gPool and SAGPool) have to learn  extra GNN / GCN layers, incurring additional overhead.

% \subsection{Discussion}
% \textbf{Relation between readout and virtual node. }

\vspace{-6pt}
\section{Experiments}
\vspace{-6pt}
We evaluate our method on graph classification and graph edit distance in single-task and multi-task settings, respectively. We conduct ablation studies on memory and time consumption. Experiments are mainly in line with the protocol of gPool~\cite{topk} and the mean and standard deviation are reported by 10-fold cross validation, except for the memory and time tests which can be estimated by one trial. We use GAT as backbone except for ablation in Tab.~\ref{tab:gcn} as analyzed in Section~\ref{sec:method} and all the experiments are conducted on one single GTX 2080 GPU.

\vspace{-6pt}
\subsection{Experiments setup}
\vspace{-6pt}
\textbf{Datasets. }D\&D~\cite{dd, wlkernel} contains graphs of protein structures. A node represents an amino acid and edges are constructed if the distance of two nodes is less than 6 $\mathring{A}$ (a unit of length in protein -- see~\cite{dd}). A label denotes whether a protein is an enzyme or a non-enzyme. PROTEINS~\cite{dd, protein} also contains proteins, where nodes are secondary structure elements. If nodes have edges, the nodes are in an amino acid sequence or in a close 3D space. NCI~\cite{nci1} is a biological dataset used for anticancer activity classification. In the dataset, each graph represents a chemical compound, with nodes and edges representing atoms and chemical bonds, respectively. NCI1 and NCI109 are commonly used for graph classification~\cite{GAT, sagpool}. FRANKENSTEIN~\cite{FRANKENSETEIN} is a set of molecular graphs~\cite{molecule} with node features containing continuous values. The label denotes whether a molecule is a mutagen or not.

\begin{table}[tb!]
    \centering
    \resizebox{\textwidth}{!}{\begin{tabular}{c|c|c|c|c|c|c}
      %  \toprule
        \toprule
        Dataset & \# Graphs & \# Classes & Avg. Nodes per Graph & Avg. Edges per Graph & \# Training & \# Test  \\
        \hline
        D\&D~\cite{dd} & 1178 & 2 & 284.32 & 715.66 & 1060 & 118 \\
        PROTEINS~\cite{protein} & 1113 & 2 & 39.06 & 72.82 & 1001 & 112 \\
        NCI1~\cite{nci1} & 4110 & 2 & 29.87 & 32.3 & 3699 & 411 \\
        NCI109~\cite{nci1} & 4127& 2 & 29.68 & 32.13 & 3714 & 413 \\
        FRANKENSTEIN~\cite{FRANKENSETEIN} & 4337 & 2 & 16.9 & 17.88 & 3903 & 434 \\
        \bottomrule
       % \bottomrule
    \end{tabular}}
    \caption{Statistic information of the used datasets.}
    \vspace{-10pt}
    \label{tab:statistics}
\end{table}

\textbf{Graph classification task. }For all the five graph-classification datasets, we set the learning rate as 1e-3 with batch size 8 and use Adam optimizer~\cite{adam} with weight decay 8e-4. 

\textbf{Graph edit distance (GED) task. }\textbf{i) setup. }The GED between graphs $\mathcal{G}_1$ and $\mathcal{G}_2$ is defined as the minimum number of edit operations needed to transform $\mathcal{G}_1$ to $\mathcal{G}_2$~\cite{ged}. Typically the edit operations include add/remove/substitute nodes and edges. Computing GED is known NP-hard in general \cite{nphard}, therefore approximations are used. There are also attempts by deep graph model \cite{ged,wang2021cvpr} and we adopt the same setting. In detail, triplet pairs are constructed by editing graph (substitute and remove) edges, which is an unsupervised model. Specifically, they substitute $k_p$ edges from graph $\mathcal{G}_1$ to generate $\mathcal{G}_{1p}$, then substitute $k_n$ edges to generate $\mathcal{G}_{1n}$. By setting $k_p<k_n$, the GED between $(\mathcal{G}_1, \mathcal{G}_{1p})$ is regarded as shorter than $(\mathcal{G}_1, \mathcal{G}_{1n})$. But actually, the GED between $(\mathcal{G}_1, \mathcal{G}_{1p})$ can be smaller than $(\mathcal{G}_1, \mathcal{G}_{1n})$ due to symmetry and isomorphism. However, the probability of such cases is typically low and decreases rapidly with increasing graph sizes. \textbf{ii) evaluation metrics. }In line with \cite{ged}, our trained backbone is evaluated by two metrics: 1) pair AUC - the area under the ROC curve for classifying pairs of graphs as similar or not and 2) triplet accuracy - the accuracy of correctly assigning higher similarity to the positive pair than the negative pair, in a triplet. 

\textbf{Compared baselines. }Since our method aims to reduce the graph size hierarchically, we mainly compare with recent hierarchical graph pooling methods: DiffPool~\cite{diffpooling}, gPool~\cite{topk} and SAGPool~\cite{sagpool} as they can readily allow for node dropping in their methods. Note we do not compare with other recent methods e.g. GIN~\cite{GIN}, GraphSAGE~\cite{graphsage}, as the methodology and motivation are different from our node dropping-based scalability-purposed methodology.

We also compare with our degenerated baseline, namely using the same backbone as DOTIN (i.e. GAT) but without node drop. The metrics include both accuracy and training time. For comprehensive comparison, on graph classification, we further compare global pooling methods Set2Set~\cite{set2set}, SortPool~\cite{sortpool}, SAGPool~\cite{sagpool}, for which node dropping cannot be (easily) fulfilled. Note that we don't directly compare DOTIN with DropGNN~\cite{dropgnn}, and the reason is that DropGNN is an ensemble model and its variant single model can be thought one of our baseline (random drop), which is in-depthly compared in ablation studies (see Fig.~\ref{fig:ratio}).

%For graph classification, we additionally compare with global pooling methods Set2Set. Note that the global pooling methods are not in consideration to compare, since they do not reduce the training cost and graph size.

%\vspace{-6pt}
\textbf{Single-task setting}
%\vspace{-6pt}
We first evaluate our method on graph classification (in Table~\ref{tab:cls}) and graph edit distance (in Table~\ref{tab:ged}) respectively. \textbf{For graph classification}, DOTIN outperforms on most of the datasets among the compared hierarchical methods. Compared with the most similar method gPool~\cite{topk}, DOTIN outperforms it on all the five datasets. Although global pooling methods (SAGPool and SortPool) get higher accuracy on NCI dataset, they only perform pooling after the final GCN layer, and the running time and complexity won't decrease but increase compared with the baseline (w/o pooling). For graph edit distance task, we report the mean accuracy and AUC scores of 10 folds in Table~\ref{tab:ged}, which is similar to graph classification. \textbf{For GED task}, DOTIN outperforms gPool~\cite{topk} with a large range, and we guess the reason is that DOTIN directly adds the regularization on the virtual node. Then, the virtual node is used for node selection, where it tends to select GED-task-irrelevant nodes to drop. In contrast, gPool~\cite{topk} does not add task-specific regularization on the designed vector, which may wrongly drop GED-task-important nodes. Besides, for both two tasks, DOTIN (w/ pooling) achieves higher accuracy and is more stable than w/o pooling, which is because the task-irrelevant pooling method reverses as much information as possible, while DOTIN (w/o pooling) simply delete the sub-graph, losing some useful information.

\textbf{Multi-task setting}
%\vspace{-6pt}
%\textbf{Experiments setup. }
We combine classification and GED, summing up their two objectives with equal weights (see also Fig.~\ref{fig:multitask}) for graph learning. For gPool~\cite{topk} and SAGPool~\cite{sagpool}, they only extract one graph representation, so we directly add the two objective regularizations on the embedding. For DOTIN, we use two virtual nodes (for two tasks) to extract the respective task-relevant graph-level information. The drop ratio is set to [0.1, 0.2, $\cdots$, 0.9]. We randomly split each dataset into 10 folds and we set batch size 16 with 5 epochs in each fold. \textbf{Results. }Table~\ref{tab:multi-task} reports the mean accuracy of 10 folds, which show DOTIN performs best in multi-task setting. For gPool and SAGPool, both classification and GED performances drop notably. Specifically, for D\&D dataset, accuracy of DOTIN on multi-task settings and single-task settings only drop 0.09\% and 0.39\% accuracy on classification and GED, respectively. For other methods, they generally drop 2 $\sim$ 5\% accuracy. We conjecture DOTIN's performance advantage is because the designed multiple virtual nodes extract and decouple different task-relevant global information to handle the corresponding task. 
%We guess that's because these methods extract a single global representation, while sometimes, different tasks can be contradictory. Hence, it can not converge to one global embedding, which accommodates two tasks simultaneously. While for DOTIN, the designed multiple virtual nodes extract and decouple different task-relevant global information to accommodate the corresponding task, resulting in higher accuracy of both two tasks. 
\vspace{-6pt}
\begin{table}[tb!]
    \centering
    \resizebox{0.72\textwidth}{!}{\begin{tabular}{c|c|c|c|c|c|c}
        %\toprule
        \toprule
        ~ & Models & D\&D & PROTEINS & NCI1 & NCI109 & FRANKENSETEIN \\
        \hline
        \multirow{3}*{\rotatebox{90}{Global}} & Set2Set~\cite{set2set} & 71.27 $\pm$ 0.84 & 66.06 $\pm$ 1.66 & 68.55 $\pm$ 1.92 & 69.78 $\pm$ 1.16 & 61.92 $\pm$ 0.73 \\
        ~ & SortPool~\cite{sortpool} & 72.53 $\pm$ 1.19 & 66.72 $\pm$ 3.56 & 73.82 $\pm$ 0.96 & 74.02 $\pm$ 1.18 & 60.61 $\pm$ 0.77 \\
        ~ & SAGPool~\cite{sagpool} & \textbf{76.19 $\pm$ 0.94} & \textbf{70.04 $\pm$ 1.47} & \textbf{74.18 $\pm$ 1.20} & \textbf{74.06 $\pm$ 0.78} & \textbf{62.57 $\pm$ 0.60} \\
        \hline
        \multirow{5}*{\rotatebox{90}{Hierarchical}} & DiffPool~\cite{diffpooling} & 66.95 $\pm$ 2.41. & 68.20 $\pm$ 2.02 & 62.32 $\pm$ 1.90 & 61.98 $\pm$ 1.98 & 60.60 $\pm$ 1.62 \\
        ~ & gPool~\cite{topk} & 75.01 $\pm$ 0.86 & 71.10 $\pm$ 0.90 & 67.02 $\pm$ 2.25 & 66.12 $\pm$ 1.60 & 61.46 $\pm$ 0.84 \\
        ~ & SAGPool~\cite{sagpool} & 76.45 $\pm$ 0.97 & 71.86 $\pm$ 0.97 & 67.45 $\pm$ 1.11 & \textbf{67.86 $\pm$ 1.41} & 61.73 $\pm$ 0.76 \\
        ~ & DOTIN (w/o pooling) & 77.41 $\pm$ 0.72 & 73.50 $\pm$ 0.5 & 68.12 $\pm$ 1.07 & 67.27 $\pm$ 1.12 & 62.41 $\pm$ 1.12 \\
        ~ & DOTIN (w/ pooling) & \textbf{78.25 $\pm$ 0.61} & \textbf{74.63 $\pm$ 0.37} & \textbf{69.39 $\pm$ 1.02} & 67.62 $\pm$ 1.04 & \textbf{63.01 $\pm$ 0.92} \\
       % \bottomrule
        \bottomrule
    \end{tabular}}
    \caption{Graph classification accuracy in 10 folds on benchmarks. %Subscript $g$ and $h$ means global and hierarchical pooling. The top and bottom half are pooling and dropping based methods, respectively.
    }
    \label{tab:cls}
    \vspace{-10pt}
\end{table}

\begin{table}[tb!]
    \centering
    \vspace{-5pt}
    \resizebox{0.72\textwidth}{!}{\begin{tabular}{c|c|c|c|c|c|c}
     %   \toprule
        \toprule
        \multirow{2}*{Model} & \multicolumn{2}{c|}{D\&D} & \multicolumn{2}{c|
        }{PROTEINS} & \multicolumn{2}{c}{NCI1} \\
        \cline{2-7}
        ~ & ACC & AUC & ACC & AUC & ACC & AUC \\
        \hline
        DiffPool~\cite{diffpooling} & 85.19 $\pm$ 0.52 & 72.29 $\pm$ 1.64 &71.12 $\pm$ 1.08 & 54.44 $\pm$ 2.99 & 80.68 $\pm$ 1.71 & 59.98 $\pm$ 3.38 \\
        gPool~\cite{topk} & 88.26 $\pm$ 0.61 & 72.89 $\pm$ 1.66 & 72.17 $\pm$ 1.02 & 54.91 $\pm$ 2.27 & 85.54 $\pm$  1.29 & 62.29 $\pm$ 2.98 \\
        SAGPool~\cite{sagpool} & 90.17 $\pm$ 0.44 & 73.01 $\pm$ 1.29 & 73.31 $\pm$ 1.14 & 59.18 $\pm$ 2.46 & 85.92 $\pm$ 1.19 & 64..48 $\pm$ 2.70 \\
        DOTIN (w/o pooling) & 90.67 $\pm$ 0.31 & 73.09 $\pm$ 1.22 & 75.89 $\pm$ 0.69 & 58.81 $\pm$ 2.19 & 87.34 $\pm$ 0.84 & 65.45 $\pm$ 2.61 \\
        DOTIN (w/ pooling) & \textbf{91.88 $\pm$ 0.24} & \textbf{74.27 $\pm$ 1.13} & \textbf{76.76 $\pm$ 0.57} & \textbf{59.91 $\pm$ 2.01} & \textbf{88.38 $\pm$ 0.73} & \textbf{66.83 $\pm$ 2.28} \\
      %  \bottomrule
        \bottomrule
    \end{tabular}}
    \caption{Accuracy (ACC) and AUC score in 10 folds of solving graph edit distance on benchmarks.}
    \label{tab:ged}
    \vspace{-10pt}
\end{table}
\begin{table}[tb!]
    \centering
    \resizebox{0.98\textwidth}{!}{\begin{tabular}{c|c|c|c|c|c|c|c|c|c}
      %  \toprule
        \toprule
        \multirow{2}*{Model} & \multicolumn{3}{c|}{D\&D} & \multicolumn{3}{c|}{PROTEINS} & \multicolumn{3}{c}{NCI1}\\
        \cline{2-10} 
        ~ & CLS & GED ACC & GED AUC& CLS & GED ACC& GED AUC & CLS & GED ACC& GED AUC\\
        \hline
        DiffPool~\cite{diffpooling} & 64.67 $\pm$ 2.92 & 81.17 $\pm$ 0.77 & 68.23 $\pm$ 1.99 & 65.20 $\pm$ 2.18 & 69.47 $\pm$ 1.42 & 52.41 $\pm$ 2.06 & 59.92 $\pm$ 1.97 & 74.15 $\pm$ 1.86 & 54.16 $\pm$ 3.21 \\
        gPool~\cite{topk} & 72.06 $\pm$ 0.99 & 89.38 $\pm$ 0.68 & 69.38 $\pm$ 1.79 & 69.92 $\pm$ 0.94 & 70.25 $\pm$ 1.16 & 52.77 $\pm$ 2.09 & 64.28 $\pm$ 2.49 & 82.47 $\pm$  1.67 & 57.37 $\pm$ 3.06 \\
        SAGPool~\cite{sagpool} & 75.53 $\pm$ 1.04 & 88.31 $\pm$ 0.57 & 71.39 $\pm$ 1.44 & 70.61 $\pm$ 0.93 & 71.16 $\pm$ 1.37 & 57.18 $\pm$ 2.16 & 68.19 $\pm$ 1.32 & 81.92 $\pm$ 2.24 & 59.94 $\pm$ 2.99 \\
        DOTIN (w/o pooling) & 77.36 $\pm$ 0.78 & 90.49 $\pm$ 0.38 & 73.01 $\pm$ 1.21 & 73.41 $\pm$ 0.62 & 75.24 $\pm$ 0.57 & 58.44 $\pm$ 2.21 & 67.99 $\pm$ 1.14 & 86.99 $\pm$ 0.92 & 64.56 $\pm$ 2.73 \\
        DOTIN (w/ pooling) & \textbf{78.14 $\pm$ 0.72} & \textbf{91.49 $\pm$ 0.31} & \textbf{74.16 $\pm$ 1.12} & \textbf{74.43 $\pm$ 0.49} & \textbf{76.16 $\pm$ 0.42} & \textbf{59.16 $\pm$ 2.03} & \textbf{68.92 $\pm$ 1.09} & \textbf{87.97 $\pm$ 0.89} & \textbf{66.29 $\pm$ 2.24} \\
       % \bottomrule
        \bottomrule
    \end{tabular}}
    \caption{Classification accuracy and GED AUC/accuracy under multi-task setting.}
    \label{tab:multi-task}
    \vspace{-10pt}
\end{table}

%To further demonstrate the efficiency of DOTIN, 
%\vspace{-6pt}
\subsection{Ablation study}
\vspace{-6pt}
%We guess for graph classification, a few nodes can already classify the overall graphs and DOTIN deletes some task-irrelevant and duplicates nodes. 
\textbf{Drop ratio effect.} We conduct ablation studies on the drop ratio $\alpha$. We construct three layers in the backbone and change the drop ratio from 0.1 to 0.9. We fix the hidden dimension as 512 and set the linear layers' connection dropout rate~\cite{dropout} 0.2 for inference (different from the meaning of the proposed node drop rate). We compare with baseline (w/o drop) and random drop. Fig.~\ref{fig:ratio} and Fig.~\ref{fig:time} show the classification accuracy and training speed (batch per second) of DOTIN with different drop ratios. Note that for random drop, the performance variance increases notably with higher drop ratio, but DOTIN is more stable even 90\% of the nodes drop. For low drop ratio, random drop plays a role of regularization, which seems without much accuracy drop. We also observe DOTIN occasionally outperforms baseline (w/o drop) sometimes, and we think this phenomenon may cause by node redundancy in original graphs, i.e., some task-irrelevant nodes influence the predictive accuracy. Fig.~\ref{fig:time} gives the training time (batches per second) with different drop ratios. The baseline (w/o drop) passes average 12.23 batches per second for a whole training epoch. After dropping 90\% nodes, DOTIN passes about 18.81 batches per second, which speeds up about 53.8\% with little accuracy drop (74.91\% v.s. 75.51\%).
\begin{figure}[tb!]
\centering
\subfigure[CLS accuracy v.s. drop ratio]{
\begin{minipage}[tb!]{0.325\linewidth}
\centering
\includegraphics[width=0.95\textwidth,angle=0]{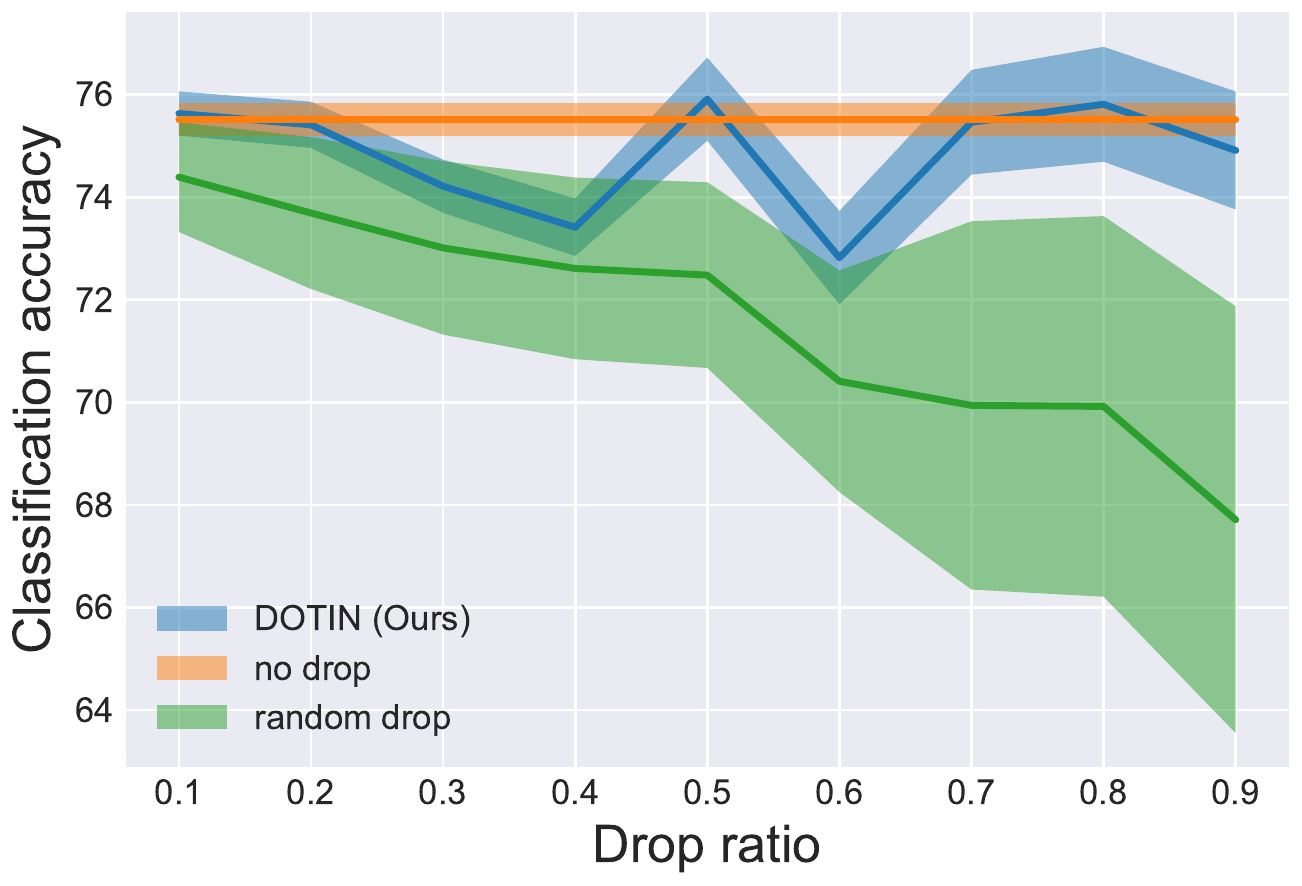}
\label{fig:ratio}
\end{minipage}%
}%
\subfigure[Speed v.s. drop ratio]{
\begin{minipage}[tb!]{0.325\linewidth}
\centering
\includegraphics[width=0.95\textwidth,angle=0]{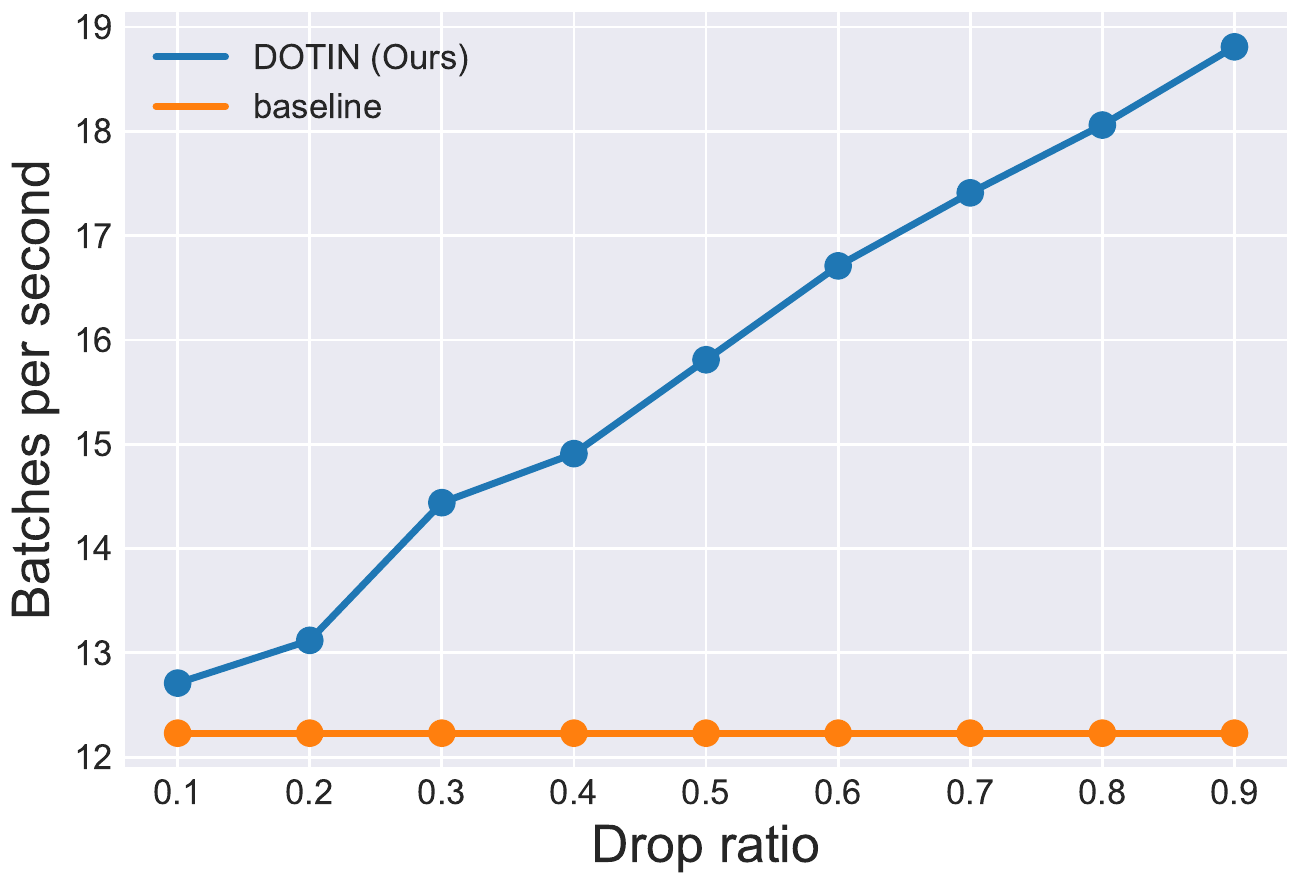}
\label{fig:time}
\end{minipage}%
}
\subfigure[Speed v.s. Num. of layers]{
\begin{minipage}[tb!]{0.325\linewidth}
\centering
\includegraphics[width=0.95\textwidth,angle=0]{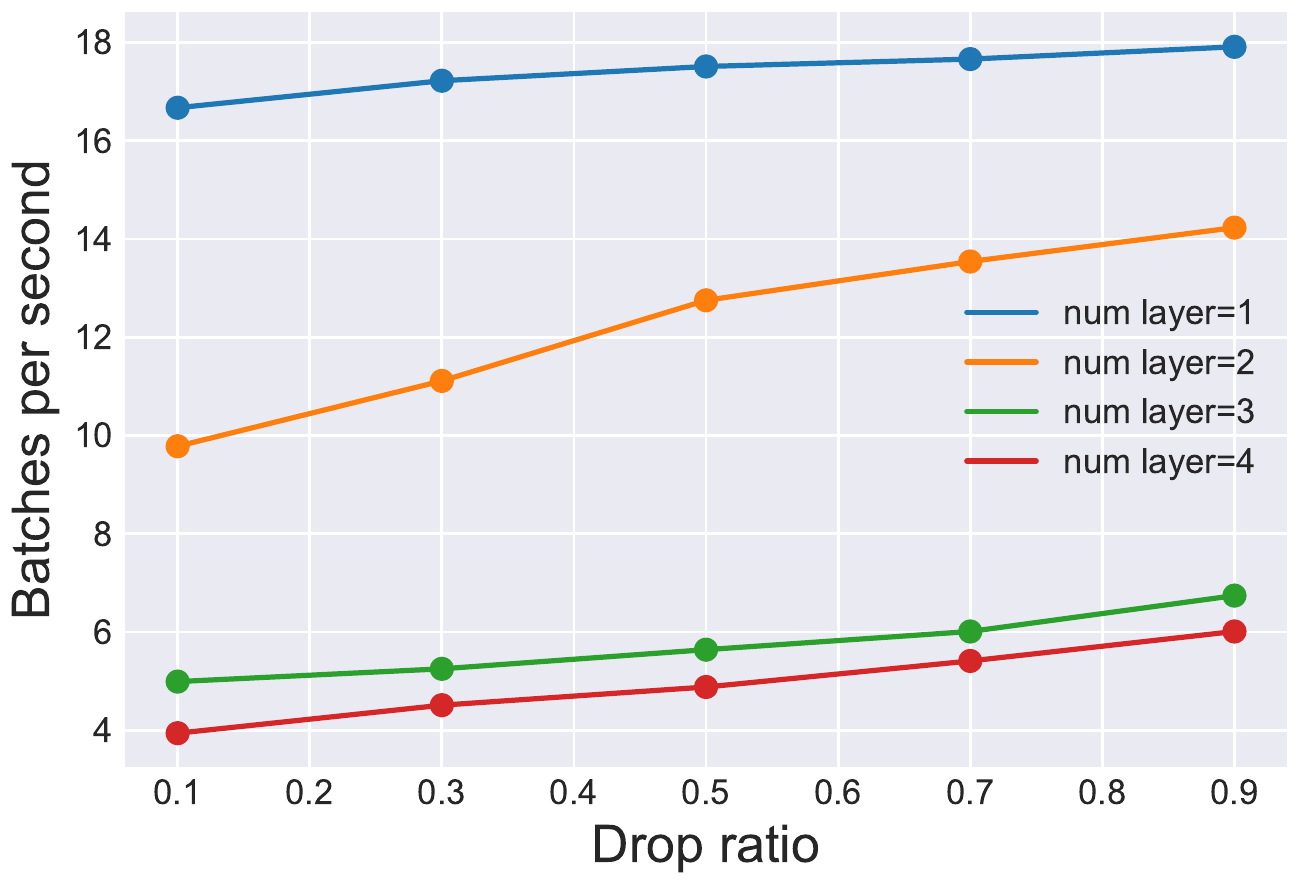}
\label{fig:layers}
\end{minipage}%
}
\vspace{-15pt}
\caption{Ablation on performance and batch training speed. Details can be seen in experiment setup.}
\vspace{-10pt}
\end{figure}
\begin{figure}[tb!]
\centering
\subfigure[Memory v.s. $num\ layers$]{
\begin{minipage}[tb!]{0.325\linewidth}
\centering
\includegraphics[width=0.95\textwidth,angle=0]{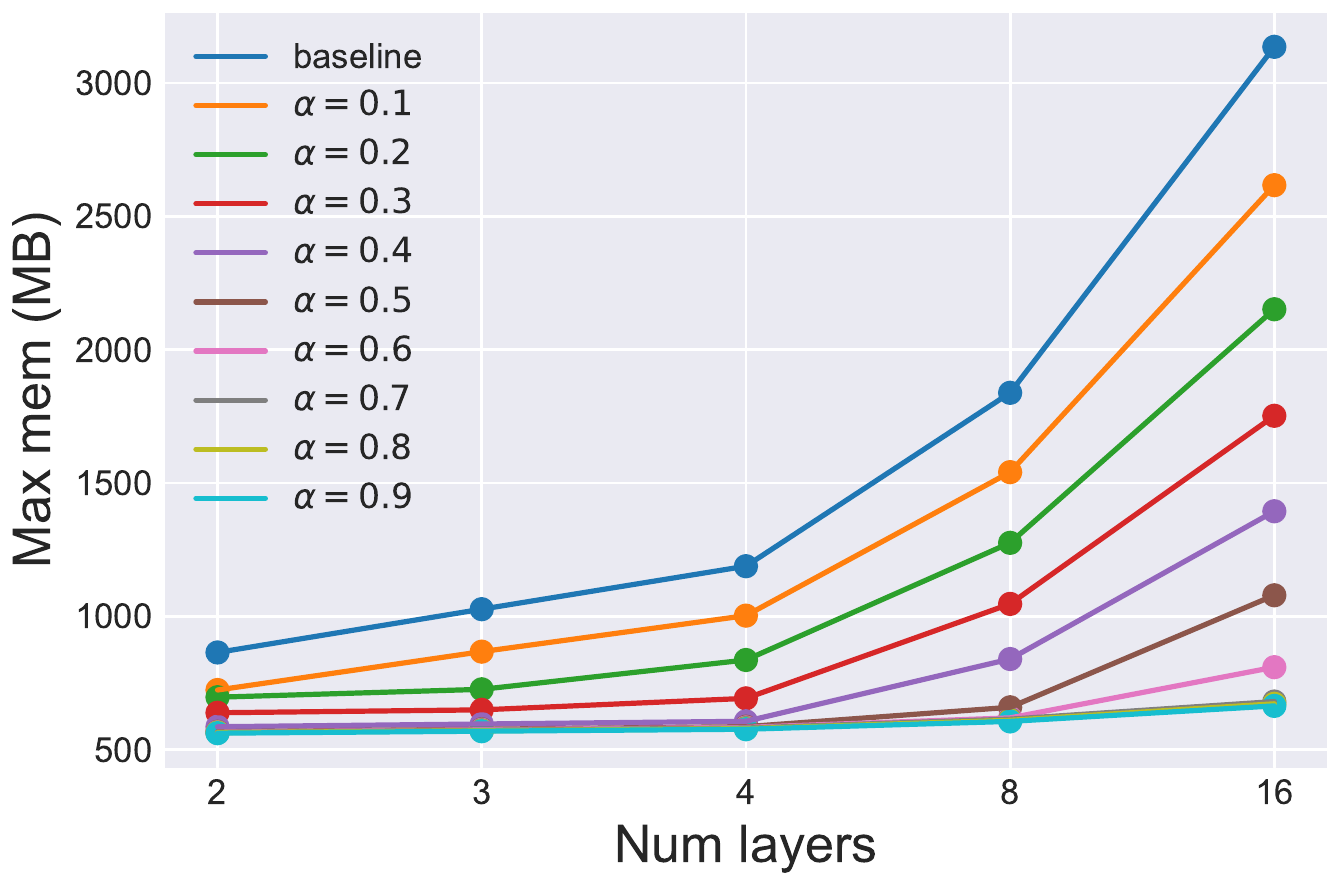}
\label{fig:storage1}
\end{minipage}%
}%
\subfigure[Memory v.s. batch size]{
\begin{minipage}[tb!]{0.325\linewidth}
\centering
\includegraphics[width=0.95\textwidth,angle=0]{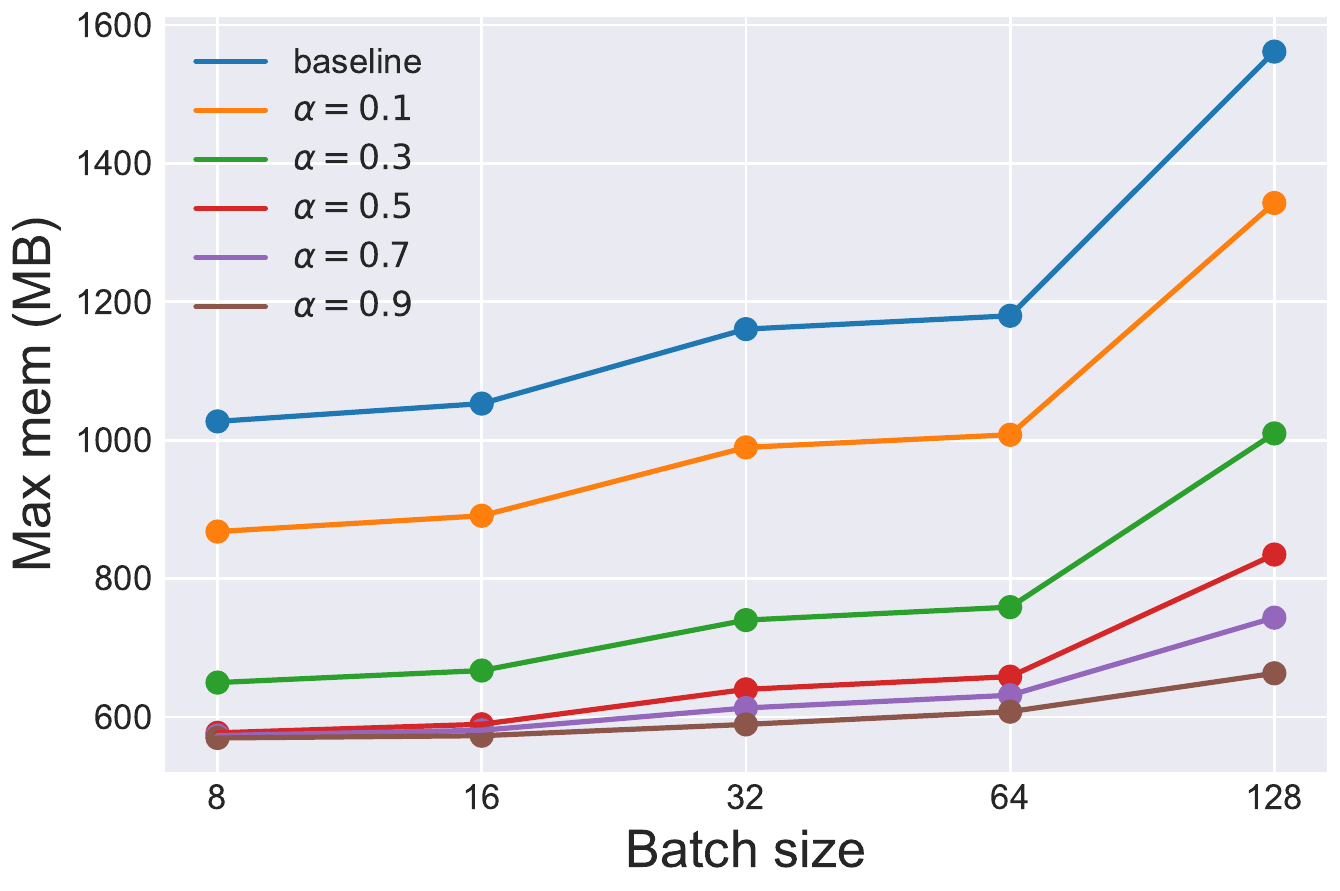}
\label{fig:storage2}
\end{minipage}%
}
\subfigure[Memory v.s. $hidden\ dim$]{
\begin{minipage}[tb!]{0.325\linewidth}
\centering
\includegraphics[width=0.95\textwidth,angle=0]{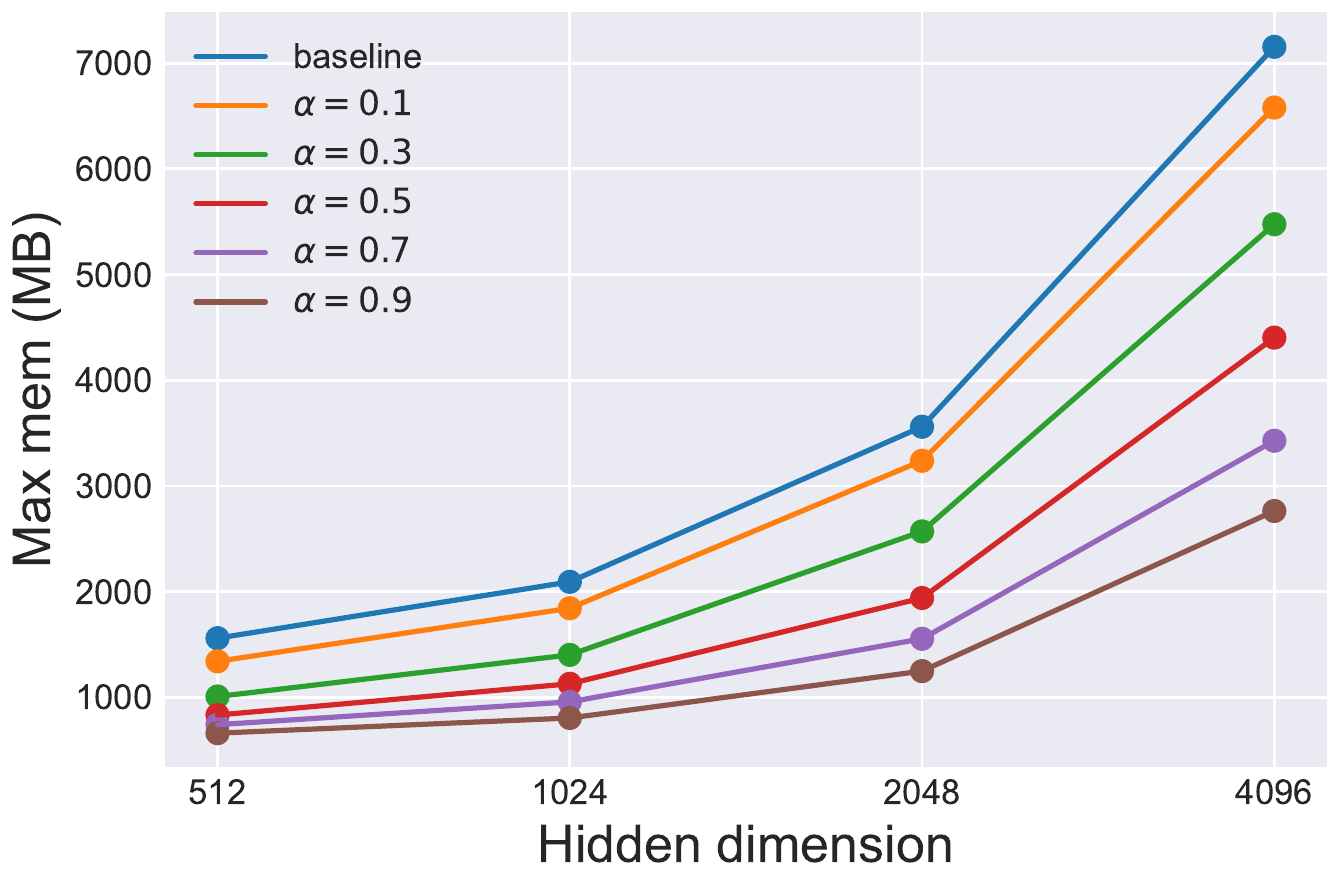}
\label{fig:storage3}
\end{minipage}%
}
\vspace{-15pt}
\caption{Ablation on memory. All the three experiments are using GAT backbone.}
\vspace{-15pt}
\end{figure}

\textbf{Number of layers. }To explore the speed rate of DOTIN with a different number of layers, we conduct experiments on single GED tasks on D\&D dataset. We fix the hidden dimension as 64 (enable to increase the number of layers) with ELU activation function~\cite{elu}. We set the Num. of layers in [1, 3, 5, 7] and drop ratio as in [0.1, 0.3, 0.5, 0.7, 0.9]. We illustrate each two combination ($num\ layers, drop\ ratio$) in Fig.~\ref{fig:layers}. It's not surprising for one layer, the efficiency reduction is not significant, and this is because we drop the nodes after the first DOTIN layer. Then, the remaining nodes only pass through one linear layer. The overall FLOPs gap is only $\alpha \times D \times D_1$, where $D$ and $D_1$ are input and output feature dimensions of the linear layer, respectively. For $num\ layer>1$, we can find DOTIN with $\alpha=0.9$ can speed up about 50\% than $\alpha=0.1$.

\textbf{Backbone. }DOTIN is agnostic to the backbone, and we integrate DOTIN into GCN~\cite{GCN} as its backbone. Different from vallina GCNs, DOTIN in GCN calculates the normalized adjacency matrix in each layer, since in each layer, we drop $\lceil N\times\alpha \rceil$ nodes, and correspondingly, the adjacency matrix will be modified. In detail, we set $num\ layers=2$ and hidden dimension as 256 for both GCN-based and GAT-based backbones. We find DOTIN in GCN-based backbone is more sensitive to drop ratio than GAT-based backbone, and we guess that's because GAT-based backbone can learn adjacency matrix by itself, which allows the classification virtual node to give higher edge weights to relevant nodes. While for GCN, the virtual node is initially fully-connected, i.e., in the message passing stage, the virtual node will take all the nodes in graph equally (all nodes contribute equally). Nevertheless, we can also observe that GCN leads to lower variance than GATs, which may be caused by the overfitting (GAT has two more linear parameters to learn in each layer, see Tab.~\ref{tab:gcn}).
%Capability is also one of the bottlenecks of deep GNNs, especially when graph meets large-scale models like transformers~\cite{vit}.
%To demonstrate DOTIN really enhances the capability of GNNs even for large-scale models, 

\textbf{Memory. }We conduct a group of experiments by increasing hidden dimension, batch size, and $num\ layer$ with drop ratio $\alpha$. We first analyze the effect of $num\ layers$, we fix batch size as 8 and hidden dimension as 512. Then, we switch the $num\ layers$ from [2, 3, 4, 8, 16] and the results are given in Fig.~\ref{fig:storage1}. We find that with 16 layers, baseline (w/o drop) spends maximal 3135.55MB memory, while for $\alpha=0.9$, DOTIN only spends maximal 664.47MB memory, which is about 22\% of baseline. However, as illustrated in Fig.~\ref{fig:ratio}, even DOTIN drops 90\% nodes, the performance almost doesn't decrease. For Fig.~\ref{fig:storage2}, we fix the hidden dimension as 512 and $num\ layers$ as 3. Then, we change the batch size from 8 to 128 and drop ratios $\alpha$ from 0.1 to 0.9. For Fig.~\ref{fig:storage3}, we fix batch size as 128 and $num\ layers$ as 3, and increase the hidden dimension from 512 to 4096. For $hidden\ dim=512$, DOTIN with drop ratio 0.9 only spends 42\% memory of baseline (w/o drop) and for $hidden\ dim=4096$, DOTIN with 90\% drop ratio reduces 4390MB memory over baseline.

\begin{figure}[tb!]
\centering
\subfigure{
\begin{minipage}[tb!]{0.235\linewidth}
\centering
\includegraphics[width=0.95\textwidth,angle=0]{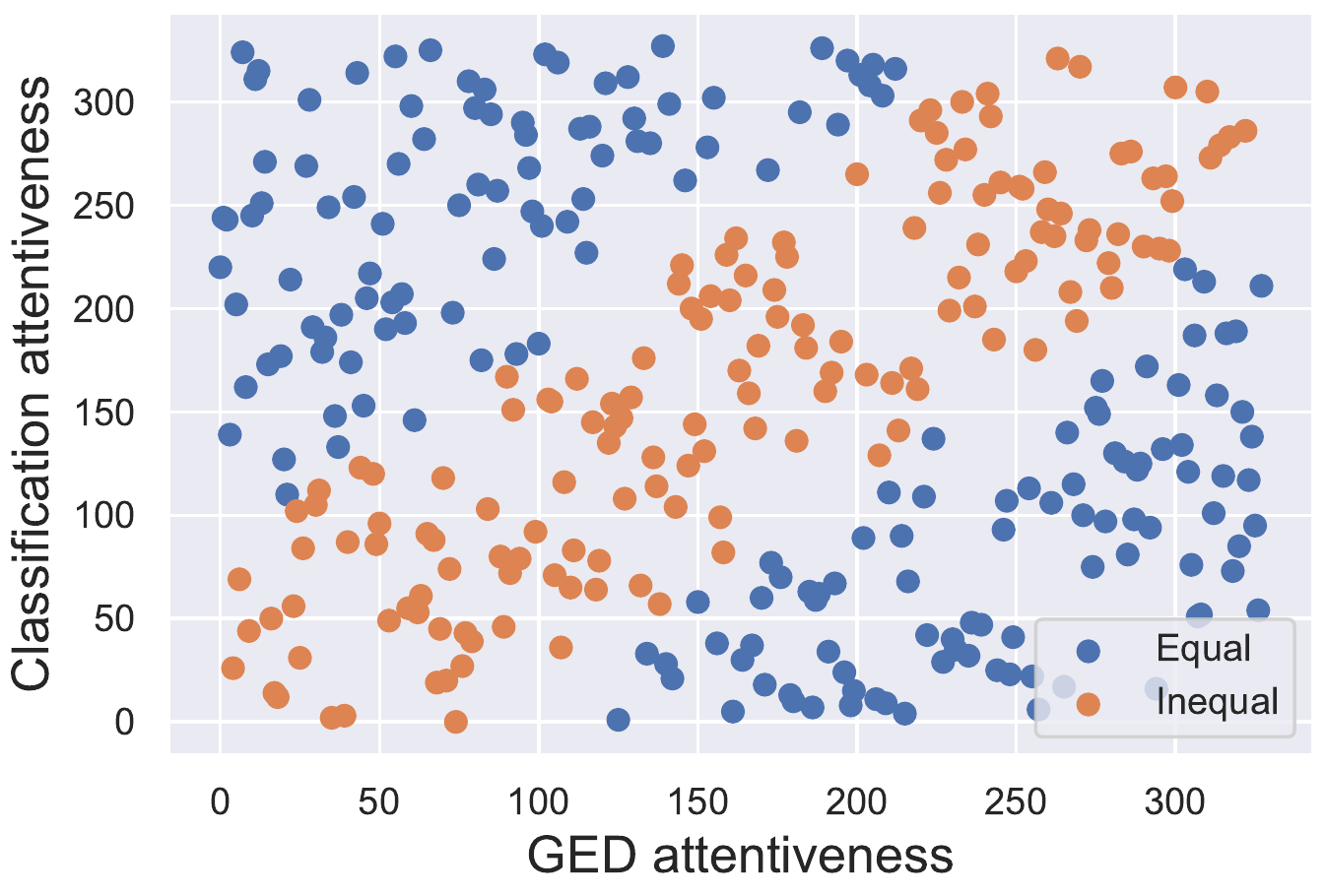}
\label{fig:dda}
\end{minipage}%
}%
\subfigure{
\begin{minipage}[tb!]{0.235\linewidth}
\centering
\includegraphics[width=0.95\textwidth,angle=0]{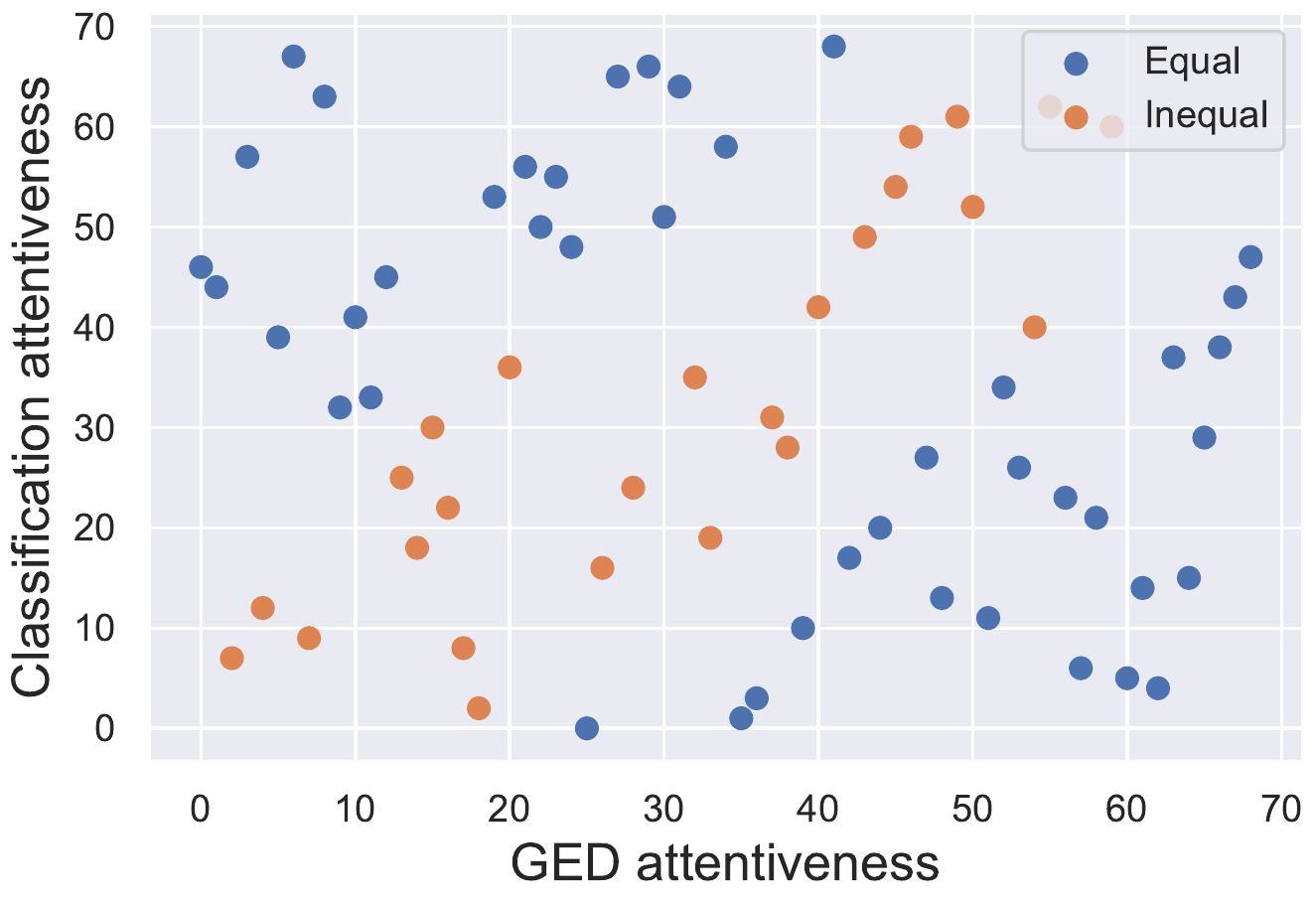}
\label{fig:ddb}
\end{minipage}%
}
\subfigure{
\begin{minipage}[tb!]{0.235\linewidth}
\centering
\includegraphics[width=0.95\textwidth,angle=0]{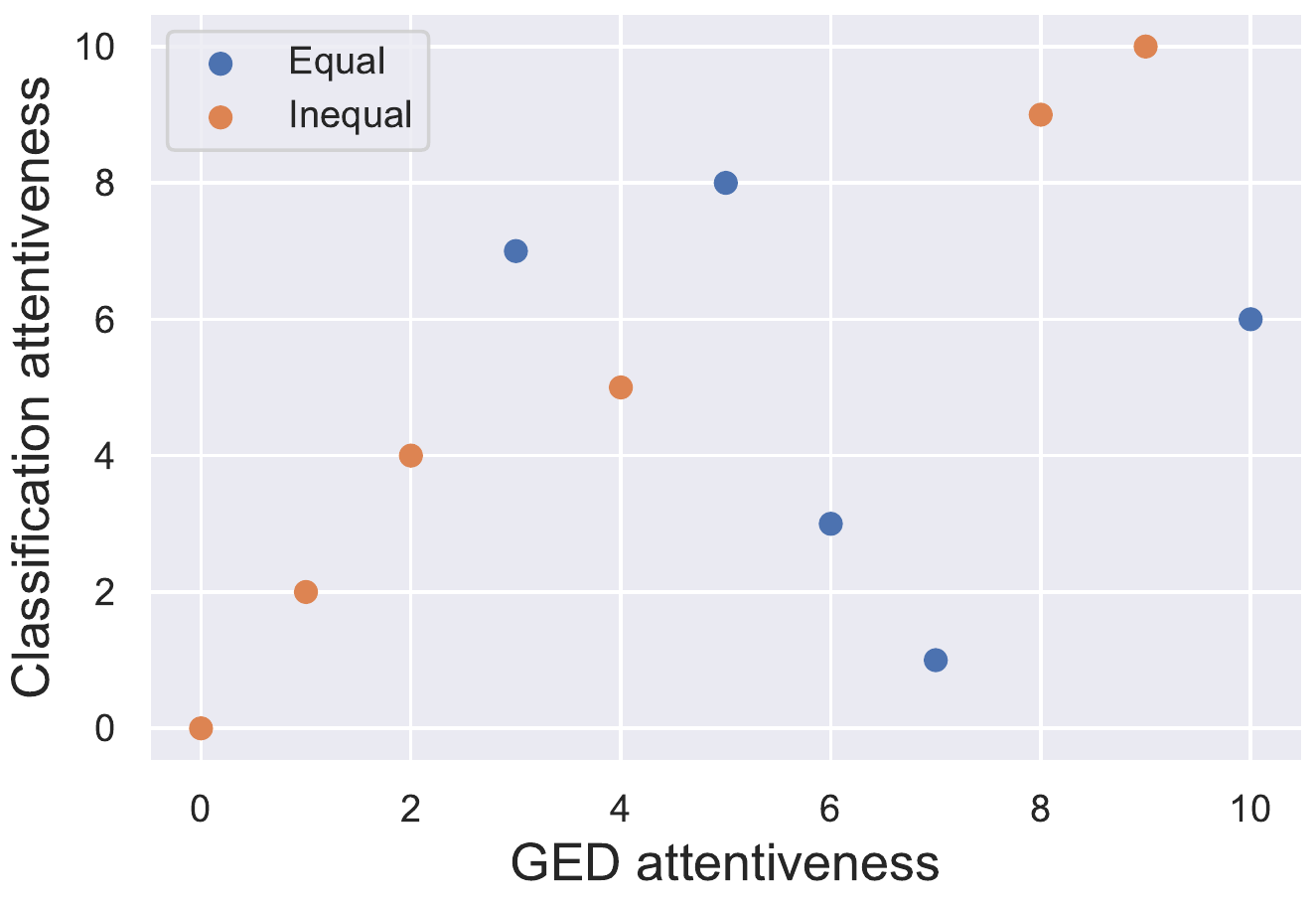}
\label{fig:proa}
\end{minipage}%
}
\subfigure{
\begin{minipage}[tb!]{0.235\linewidth}
\centering
\includegraphics[width=0.95\textwidth,angle=0]{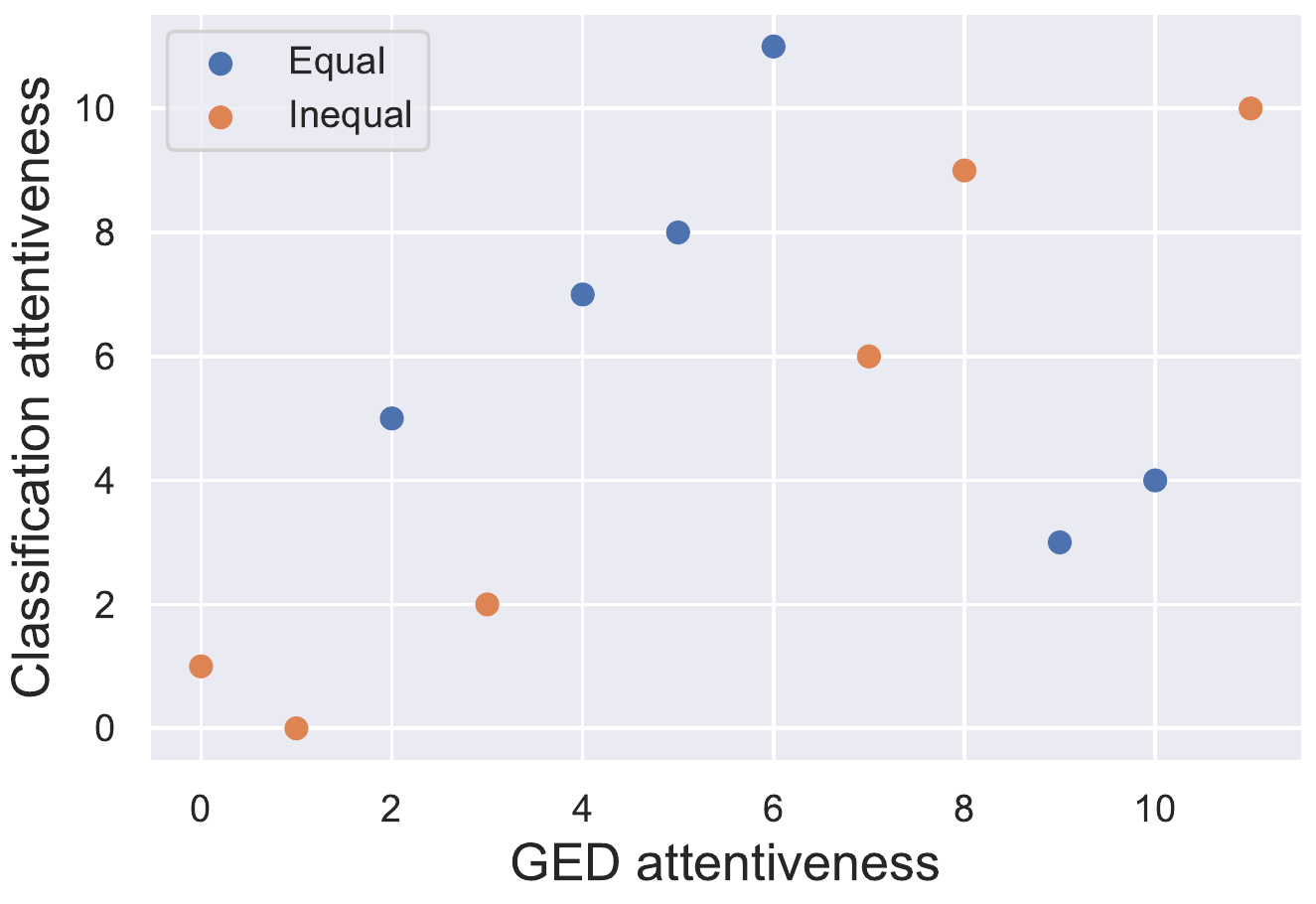}
\label{fig:prob}
\end{minipage}%
}
\vspace{-10pt}
\caption{Scatter plot for importance ranking on D\&D (left two examples) and PROTEINS (right two examples). The $x$ axis is the ranked importance for GED and $y$ for classification. Points close to diagonal (in orange) mean similar importance to two tasks (for orange, the attentive score ratio of the two tasks is at least 0.25), while those in blue indicate relevant to one task but less to the other.}
\vspace{-15pt}
\end{figure}

\begin{figure}[tb!]
\subfigure{
\begin{minipage}[tb!]{0.235\linewidth}
\centering
\includegraphics[width=0.95\textwidth,angle=0]{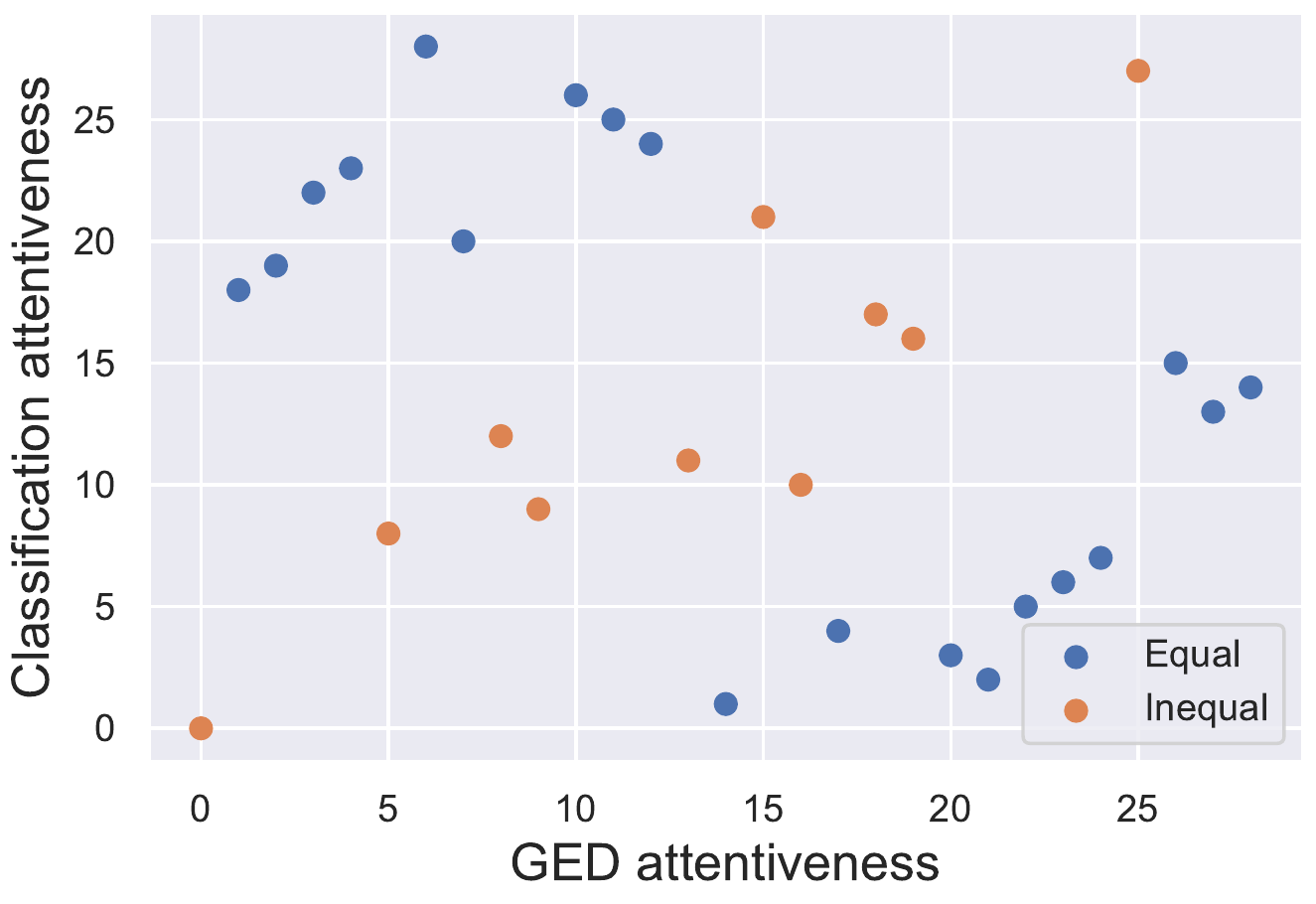}
\label{fig:nci1a}
\end{minipage}%
}%
\subfigure{
\begin{minipage}[tb!]{0.235\linewidth}
\centering
\includegraphics[width=0.95\textwidth,angle=0]{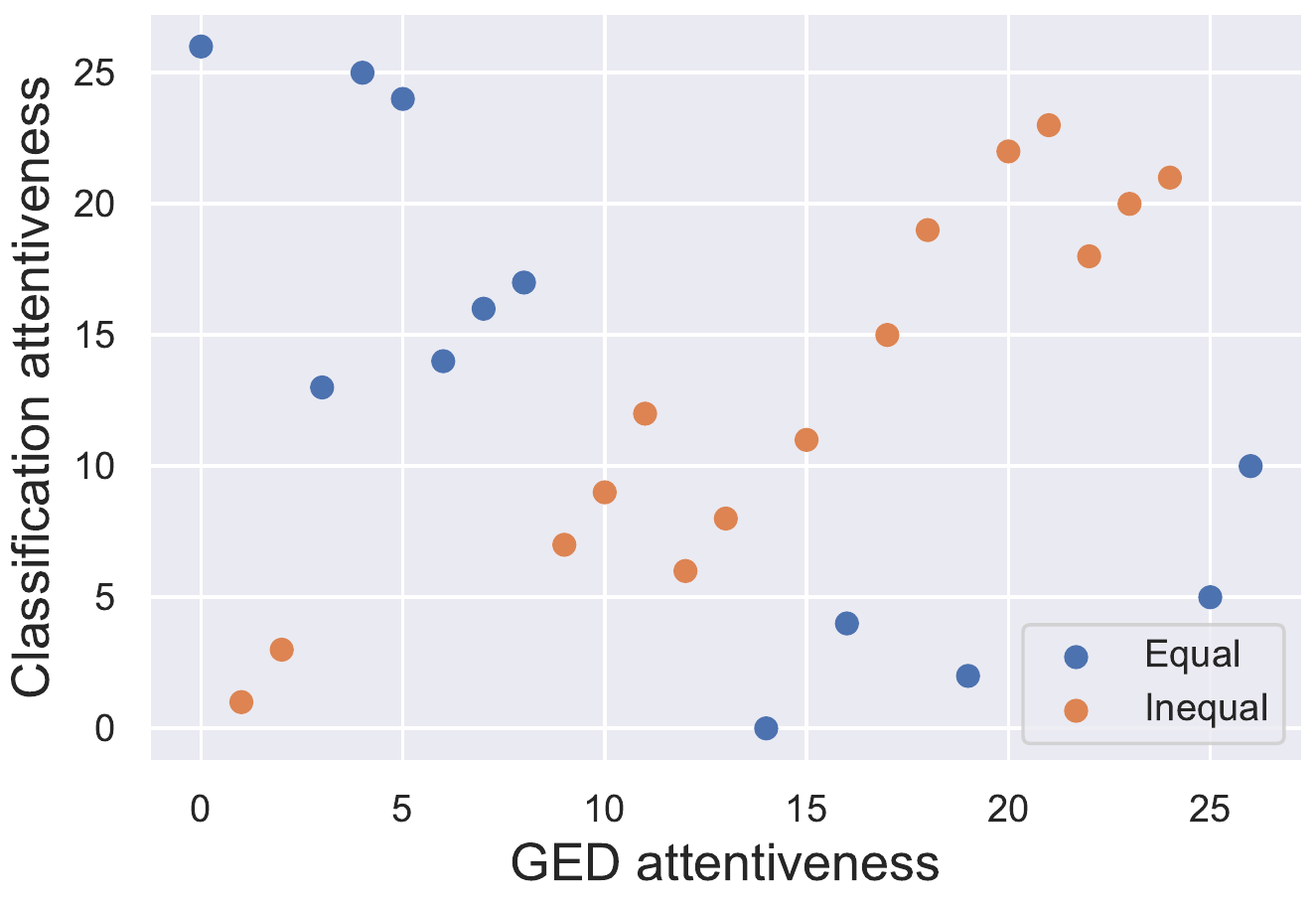}
\label{fig:nci1b}
\end{minipage}%
}
\subfigure{
\begin{minipage}[tb!]{0.235\linewidth}
\centering
\includegraphics[width=0.95\textwidth,angle=0]{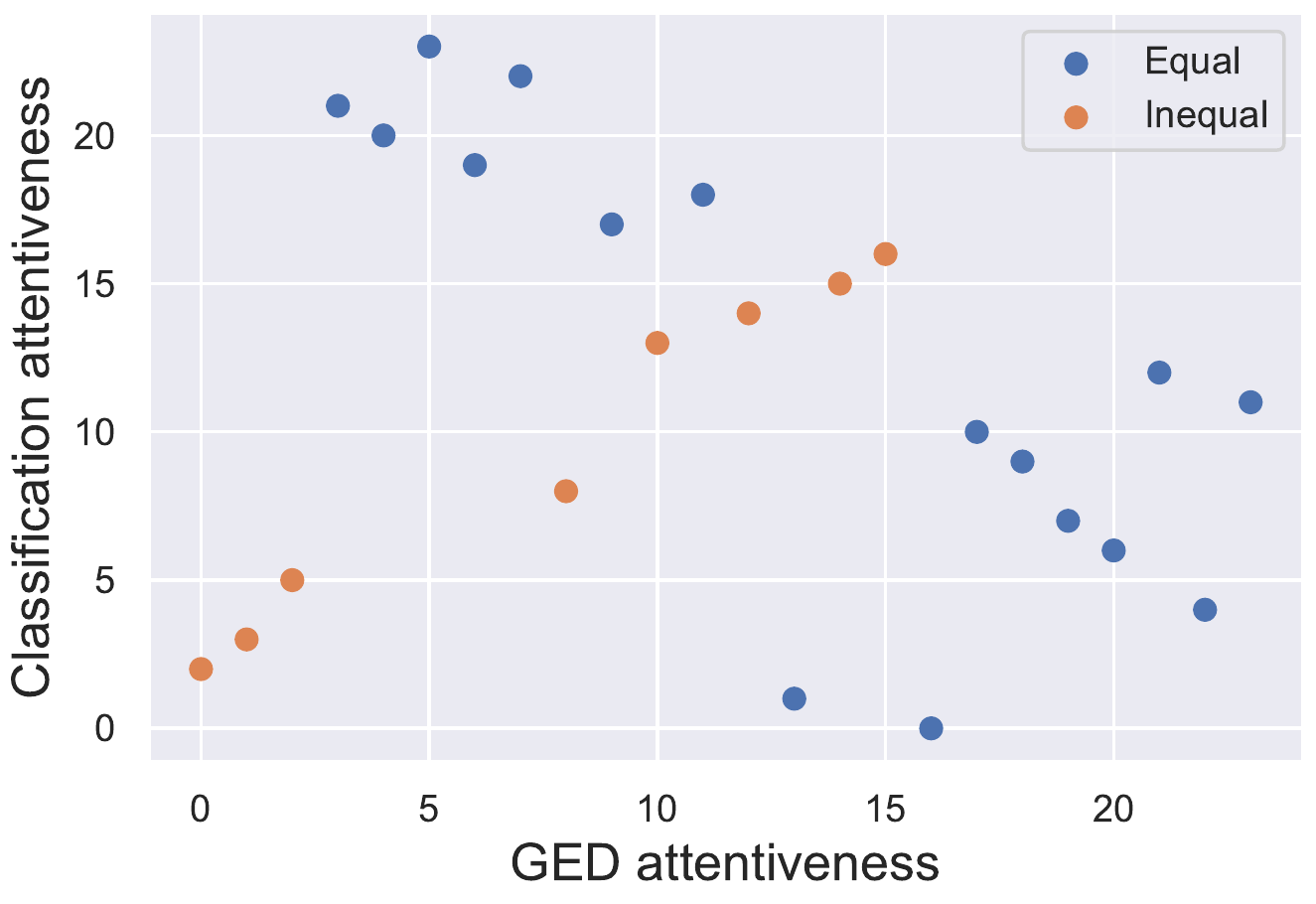}
%\caption{NCI1 }
\label{fig:nci109a}
\end{minipage}%
}
\subfigure{
\begin{minipage}[tb!]{0.235\linewidth}
\centering
\includegraphics[width=0.95\textwidth,angle=0]{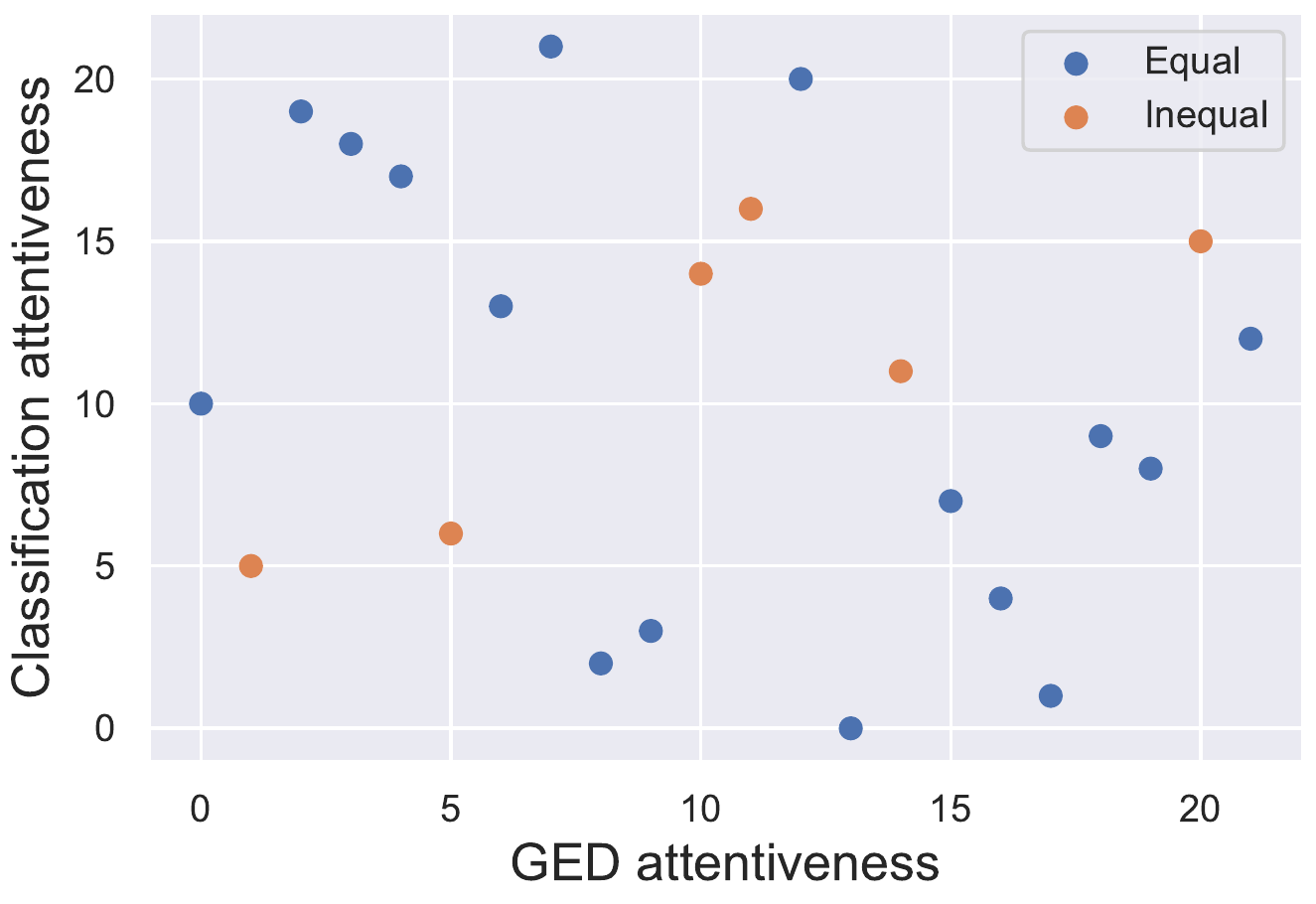}
\label{fig:nci109b}
\end{minipage}%
}
\vspace{-10pt}
\caption{Importance ranking on NCI1 (left two examples) and NCI109 (right two examples) datasets.}
\vspace{-2pt}
\end{figure}

\begin{table}[tb!]
% \vspace{-5pt}
    \centering
    \resizebox{0.72\textwidth}{!}{\begin{tabular}{c|c|c|c|c|c}
        \toprule
    %    \toprule
        Backbone & D\&D & PROTEINS & NCI1 & NCI109 & FRANKENSETEIN \\
        \hline
        GCN ($\alpha=0.1$) & 73.69 ± 0.32 & 73.29 ± 0.37 & 66.18 ± 0.89 & 65.16 ± 1.14 & 60.21 ± 1.01 \\
        GCN ($\alpha=0.5$) & 72.81 ± 0.48 & 72.51 ± 0.64 & 64.93 ± 1.03 & 63.39 ± 1.51 & 58.19 ± 1.27 \\
        GCN ($\alpha=0.9$) & 71.17 ± 0.66 & 71.14 ± 0.71 & 63.37 ± 1.24 & 61.74 ± 1.79 & 54.49 ± 1.38 \\
        \hline
        GAT ($\alpha=0.1$) & 75.63 ± 0.43 & \textbf{73.41 ± 0.48} & \textbf{67.91 ± 1.09} & \textbf{67.11 ± 1.28} & \textbf{62.18 ± 1.17} \\
        GAT ($\alpha=0.5$) & \textbf{76.41 ± 0.69} & 73.28 ± 0.61 & 67.58 ± 1.26 & 66.88 ± 1.67 & 61.26 ± 1.48 \\
        GAT ($\alpha=0.9$) & 75.99 ± 0.82 & 72.17 ± 0.79 & 66.98 ± 1.63 & 66.15 ± 1.99 & 60.99 ± 1.61 \\
      %  \bottomrule
        \bottomrule
    \end{tabular}}
    \caption{Node classification accuracy of DOTIN's variants with different backbones and drop ratios.}
    \label{tab:gcn}
\vspace{-15pt}
\end{table}
\textbf{Attentiveness for different tasks. }As illustrated in Fig.~\ref{fig:illustration}, node importance can be different for different tasks. We visualize the attentiveness of classification and GED tasks from Fig.~\ref{fig:dda} to Fig.~\ref{fig:nci109b}. For each dataset, we randomly select two graphs in test set for visualization. For training, we set hidden dimension as 256 and batch size as 8. For better comparison, we normalize the vectors before calculating the attention matrix. Then, we choose the first two rows (classification and GED virtual nodes) and sort by values. Then, we visualize the ranked number of the attentiveness. For each subplot, $x$ axis is the GED importance rank and $y$ axis classification importance rank. We can easily observe, the attentiveness for the two nodes with others are completely in different distribution, i.e., for different tasks, the \textit{task-important} nodes are different (lots of points are closed to $x$ or $y$ axis). This phenomenon also explains DOTIN outperforms gPool~\cite{topk}, since gPool may drop \textit{task-relevant} nodes, while DOTIN balances multiple tasks and drops nodes with averagely lower attentiveness. One of another interesting phenomenon is some nodes have the similar attentiveness to virtual nodes, we guess that is because of the \textit{over smoothing}~\cite{oversmoothing}, which is a classical problems in deep GNNs.

% \vspace{-6pt}
% \subsection{Further discussion}
% \vspace{-6pt}% aim to extract global information by adding 
% though itself is an arxiv paper with 5 pages only
% \textbf{Relation to VCN~\cite{graph-virtual}. }\textbf{Connections. } VCN is the only identified work that also adds virtual nodes for graph representation learning. \textbf{Differences. }For its specific purpose, VCN uses the virtual nodes for learning graph representations while DOTIN uses virtual nodes to explicitly select which nodes are task-irrelevant and drop them. For implementation, VCN extracts the final representation by a GRU network~\cite{gru} while DOTIN extracts the final representation by attention mechanism.

% \textbf{Comparison to edge drop and other sparsification methods. } \yanr{need to cite more papers!}

\vspace{-6pt}
\section{Conclusion and Outlook}
\vspace{-6pt}
We have proposed DOTIN, which aims to enhance the efficiency and scalability of existing GNNs. Our method directly regularizes the introduced $K$ virtual nodes with learnable vectors, corresponding to $K$ tasks for learning, which helps drop \textit{task-irrelevant} nodes. We apply DOTIN in GATs, with the extra cost of only $\mathcal{O}(1)$ parameters for learning. Experiments on graph classification and graph edit distance (GED) datasets show it achieves state-of-the-art results. Furthermore, DOTIN shows more advantage w.r.t cost-efficiency in multi-task settings (joint learning of graph classification and GED). 

%DOTIN can drop 90\% nodes with little  performance gap, which also speeds up about 50\% and reduces about 78\% storage (with deep GNNs). Our technique can also be easily applied for GCN-liked architecture, resulting in comparable results for single-task settings.

For future work, we aim to extend our experiments to more graph tasks beyond the current two-task setting, by exploring more  graph-level tasks which in fact is so far seldom explored in GNNs literature (most graph-level learning works are focused on the two tasks: GED and classification as studied in this paper). This effort may also help better explore DOTIN's ability in multi-task learning, as well as its generalization ability to unseen tasks or new datasets, by considering node dropping as a certain way of improving generalization, beyond cost-efficiency and scalability as focused by this paper. 

\vspace{-6pt}
\section{Limitation and Potential Negative Social Impact}
\vspace{-6pt}
\label{sec:impact}
 \textbf{Limitations of the work. }Currently, DOTIN is more applicable to GAT-based models, as GAT learns edge weights by the network itself, making different virtual nodes target different tasks. While for pre-defined edges weights, multiple nodes learn {the} same embeddings, which is {contradictory} to our initial design and rationale. New edge pruning methods can be devised to address this issue.

\textbf{Potential negative societal impacts.} 
 Deep learning can be time and energy-consuming. While our methods speed up GNNs and also with much less memory, yet DOTIN may make the future GNN community toward deeper/wider architectures, which may in turn increase energy consumption.

%The mainstream of graph networks are lightweights (usually have shallow layers, \yan{perhaps restricted to} the training speed and storage). 
\clearpage
\bibliography{neurips_2022}

\begin{thebibliography}{10}

\bibitem{mincut}
F.~M. Bianchi, D.~Grattarola, and C.~Alippi.
\newblock Spectral clustering with graph neural networks for graph pooling.
\newblock In {\em ICML}, 2020.

\bibitem{protein}
K.~M. Borgwardt, C.~S. Ong, S.~Sch{\"o}nauer, S.~Vishwanathan, A.~J. Smola, and
  H.-P. Kriegel.
\newblock Protein function prediction via graph kernels.
\newblock {\em Bioinformatics}, 21(suppl\_1):i47--i56, 2005.

\bibitem{pooling1}
C.~Cangea, P.~Veli{\v{c}}kovi{\'c}, N.~Jovanovi{\'c}, T.~Kipf, and P.~Li{\`o}.
\newblock Towards sparse hierarchical graph classifiers.
\newblock {\em arXiv preprint arXiv:1811.01287}, 2018.

\bibitem{elu}
D.-A. Clevert, T.~Unterthiner, and S.~Hochreiter.
\newblock Fast and accurate deep network learning by exponential linear units
  (elus).
\newblock {\em arXiv preprint arXiv:1511.07289}, 2015.

\bibitem{molecule}
F.~Costa and K.~De~Grave.
\newblock Fast neighborhood subgraph pairwise distance kernel.
\newblock In {\em ICML}, 2010.

\bibitem{molecules}
H.~Dai, B.~Dai, and L.~Song.
\newblock Discriminative embeddings of latent variable models for structured
  data.
\newblock In {\em ICML}, 2016.

\bibitem{dd}
P.~D. Dobson and A.~J. Doig.
\newblock Distinguishing enzyme structures from non-enzymes without alignments.
\newblock {\em Journal of molecular biology}, 2003.

\bibitem{vit}
A.~Dosovitskiy, L.~Beyer, A.~Kolesnikov, D.~Weissenborn, X.~Zhai,
  T.~Unterthiner, M.~Dehghani, M.~Minderer, G.~Heigold, S.~Gelly, et~al.
\newblock An image is worth 16x16 words: Transformers for image recognition at
  scale.
\newblock {\em arXiv preprint arXiv:2010.11929}, 2020.

\bibitem{topk}
H.~Gao and S.~Ji.
\newblock Graph u-nets.
\newblock In {\em ICML}, 2019.

\bibitem{understanding}
D.~Grattarola, D.~Zambon, F.~M. Bianchi, and C.~Alippi.
\newblock Understanding pooling in graph neural networks.
\newblock {\em arXiv preprint arXiv:2110.05292}, 2021.

\bibitem{graphsage}
W.~Hamilton, Z.~Ying, and J.~Leskovec.
\newblock Inductive representation learning on large graphs.
\newblock {\em NeurIPS}, 2017.

\bibitem{LSTM}
S.~Hochreiter and J.~Schmidhuber.
\newblock Long short-term memory.
\newblock {\em Neural computation}, 9(8):1735--1780, 1997.

\bibitem{adam}
D.~P. Kingma and J.~Ba.
\newblock Adam: A method for stochastic optimization.
\newblock {\em arXiv preprint arXiv:1412.6980}, 2014.

\bibitem{GCN}
T.~N. Kipf and M.~Welling.
\newblock Semi-supervised classification with graph convolutional networks.
\newblock {\em arXiv preprint arXiv:1609.02907}, 2016.

\bibitem{CNN}
Y.~LeCun, Y.~Bengio, and G.~Hinton.
\newblock Deep learning.
\newblock {\em nature}, 521(7553):436--444, 2015.

\bibitem{sagpool}
J.~Lee, I.~Lee, and J.~Kang.
\newblock Self-attention graph pooling.
\newblock In {\em ICML}, 2019.

\bibitem{SAPool}
H.~Lei, N.~Akhtar, and A.~Mian.
\newblock Octree guided cnn with spherical kernels for 3d point clouds.
\newblock In {\em CVPR}, 2019.

\bibitem{deepergcn}
G.~Li, C.~Xiong, A.~Thabet, and B.~Ghanem.
\newblock Deepergcn: All you need to train deeper gcns.
\newblock {\em arXiv preprint arXiv:2006.07739}, 2020.

\bibitem{ged}
Y.~Li, C.~Gu, T.~Dullien, O.~Vinyals, and P.~Kohli.
\newblock Graph matching networks for learning the similarity of graph
  structured objects.
\newblock In {\em ICML}, 2019.

\bibitem{evit}
Y.~Liang, C.~Ge, Z.~Tong, Y.~Song, J.~Wang, and P.~Xie.
\newblock Not all patches are what you need: Expediting vision transformers via
  token reorganizations.
\newblock {\em arXiv preprint arXiv:2202.07800}, 2022.

\bibitem{rethinking1}
D.~Mesquita, A.~Souza, and S.~Kaski.
\newblock Rethinking pooling in graph neural networks.
\newblock {\em NeurIPS}, 2020.

\bibitem{lapool}
E.~Noutahi, D.~Beaini, J.~Horwood, S.~Gigu{\`e}re, and P.~Tossou.
\newblock Towards interpretable sparse graph representation learning with
  laplacian pooling.
\newblock {\em arXiv preprint arXiv:1905.11577}, 2019.

\bibitem{FRANKENSETEIN}
F.~Orsini, P.~Frasconi, and L.~De~Raedt.
\newblock Graph invariant kernels.
\newblock In {\em IJCAI}, 2015.

\bibitem{dropgnn}
P.~A. Papp, K.~Martinkus, L.~Faber, and R.~Wattenhofer.
\newblock Dropgnn: random dropouts increase the expressiveness of graph neural
  networks.
\newblock {\em NeurIPS}, 2021.

\bibitem{graph-virtual}
T.~Pham, T.~Tran, H.~Dam, and S.~Venkatesh.
\newblock Graph classification via deep learning with virtual nodes.
\newblock {\em arXiv preprint arXiv:1708.04357}, 2017.

\bibitem{dropedge}
Y.~Rong, W.~Huang, T.~Xu, and J.~Huang.
\newblock Dropedge: Towards deep graph convolutional networks on node
  classification.
\newblock In {\em ICLR}, 2019.

\bibitem{wlkernel}
N.~Shervashidze, P.~Schweitzer, E.~J. Van~Leeuwen, K.~Mehlhorn, and K.~M.
  Borgwardt.
\newblock Weisfeiler-lehman graph kernels.
\newblock {\em JMLR}, 12(9), 2011.

\bibitem{dropout}
N.~Srivastava, G.~Hinton, A.~Krizhevsky, I.~Sutskever, and R.~Salakhutdinov.
\newblock Dropout: a simple way to prevent neural networks from overfitting.
\newblock {\em JMLR}, 2014.

\bibitem{transformer}
A.~Vaswani, N.~Shazeer, N.~Parmar, J.~Uszkoreit, L.~Jones, A.~N. Gomez,
  {\L}.~Kaiser, and I.~Polosukhin.
\newblock Attention is all you need.
\newblock {\em NeurIPS}, 2017.

\bibitem{GAT}
P.~Veli{\v{c}}kovi{\'c}, G.~Cucurull, A.~Casanova, A.~Romero, P.~Lio, and
  Y.~Bengio.
\newblock Graph attention networks.
\newblock {\em arXiv preprint arXiv:1710.10903}, 2017.

\bibitem{set2set}
O.~Vinyals, S.~Bengio, and M.~Kudlur.
\newblock Order matters: Sequence to sequence for sets.
\newblock {\em arXiv preprint arXiv:1511.06391}, 2015.

\bibitem{nci1}
N.~Wale, I.~A. Watson, and G.~Karypis.
\newblock Comparison of descriptor spaces for chemical compound retrieval and
  classification.
\newblock {\em Knowledge and Information Systems}, 14(3):347--375, 2008.

\bibitem{wang2021cvpr}
R.~Wang, T.~Zhanag, T.~Yu, J.~Yan, and X.~Yang.
\newblock Combinatorial learning of graph edit distance via dynamic embedding.
\newblock In {\em CVPR}, 2021.

\bibitem{GIN}
K.~Xu, W.~Hu, J.~Leskovec, and S.~Jegelka.
\newblock How powerful are graph neural networks?
\newblock {\em arXiv preprint arXiv:1810.00826}, 2018.

\bibitem{oversmoothing}
C.~Yang, R.~Wang, S.~Yao, S.~Liu, and T.~Abdelzaher.
\newblock Revisiting over-smoothing in deep gcns.
\newblock {\em arXiv preprint arXiv:2003.13663}, 2020.

\bibitem{diffpooling}
Z.~Ying, J.~You, C.~Morris, X.~Ren, W.~Hamilton, and J.~Leskovec.
\newblock Hierarchical graph representation learning with differentiable
  pooling.
\newblock {\em NeurIPS}, 31, 2018.

\bibitem{nphard}
Z.~Zeng, A.~K. Tung, J.~Wang, J.~Feng, and L.~Zhou.
\newblock Comparing stars: On approximating graph edit distance.
\newblock {\em Proceedings of the VLDB Endowment}, 2009.

\bibitem{mpassing}
L.~Zhang, D.~Xu, A.~Arnab, and P.~H. Torr.
\newblock Dynamic graph message passing networks.
\newblock In {\em CVPR}, 2020.

\bibitem{sortpool}
M.~Zhang, Z.~Cui, M.~Neumann, and Y.~Chen.
\newblock An end-to-end deep learning architecture for graph classification.
\newblock In {\em AAAI}, 2018.

\bibitem{dropedge2}
C.~Zheng, B.~Zong, W.~Cheng, D.~Song, J.~Ni, W.~Yu, H.~Chen, and W.~Wang.
\newblock Robust graph representation learning via neural sparsification.
\newblock In {\em ICML}, 2020.

\end{thebibliography}
\bibliographystyle{abbrv}

\newpage
\appendix
\section{More Discussion on Related works}
\label{sec:morerelated}
\textbf{Graph backbone and aggregation functions. }Various GNN backbones~\cite{GCN, GAT, graphsage} have been devised to capture graph structural information. They mainly differ in the specific aggregation functions. As a comprehensive study, GraphSAGE~\cite{graphsage} adopts four different aggregation methods, namely max, mean, GCN~\cite{GCN}, and LSTM~\cite{LSTM}. Instead, Graph Attention Networks~\cite{GAT} proposes attention-based methods~\cite{transformer}, where the edge weights are learned by network adaptively. Graph Isomorphism Networks (GINs)~\cite{GIN} proves that GNN can satisfy the 1-Weisfeiler-Lehman (WL) condition only with sum pooling function as aggregation function. Recently, DeeperGCN~\cite{deepergcn} proposes a trainable softmax and power-mean aggregation function that generalizes basic operators.
 
\textbf{Readout function. }The readout functions are widely designed by statistics e.g. min/max/sum/average of nodes to represent the graphs~\cite{GAT, GCN}. On the basis of global pooling schemes, SortPooling~\cite{sortpool} chooses the top-$k$ values from the sorted list of the node features to construct outputs. Instead, in this paper we propose to adopt a  virtual node to help explicitly select the nodes to drop out. Interestingly, the term of \textit{virtual node} is also called and used in VCN~\cite{graph-virtual}. In that work, the virtual node together with the real nodes are fed into an RNN for extracting the embedding of the whole graph\footnote{We are a bit in short of confidence to rephrase the exact structure of the network for graph feature extraction in that 5-page arxiv paper~\cite{graph-virtual}, whereby the details are not well described, and no modern GNN structure, but instead RNN is used.}.
% \section{Limitation and Potential Negative Social Impact}
% \label{sec:impact}
%  \textbf{Limitations of the work. }Currently, DOTIN is more applicable to GAT-based models, as GAT learns edge weights by the network itself, making different virtual nodes target different tasks. While for pre-defined edges weights, multiple nodes learn {the} same embeddings, which is {contradictory} to our initial design and rationale. New edge pruning methods can be devised to address this issue.

% \textbf{Potential negative societal impacts.} 
%  Deep learning can be time and energy-consuming. While our methods speed up GNNs and also with much less memory, yet DOTIN may make the future GNN community toward deeper/wider architectures, which may in turn increase energy consumption.

%extract graph feature. Virtual node is also exadocitlptedy  inpli . VCN is the only identified work that also adds virtual nodes for graph representation learning. \textbf{Differences. }For its specific purpose, VCN uses the virtual nodes for learning graph representations while DOTIN uses virtual nodes to explicitly select which nodes are task-irrelevant and drop them. For implementation, VCN extracts the final representation by a GRU network~\cite{gru} while DOTIN extracts the final representation by attention mechanism.

\end{document}